\documentclass[10pt]{article} 

\usepackage{amsmath,amsfonts,bm}

\def\eqref#1{equation~\ref{#1}}

\def\1{\bm{1}}

\DeclareMathAlphabet{\mathsfit}{\encodingdefault}{\sfdefault}{m}{sl}
\SetMathAlphabet{\mathsfit}{bold}{\encodingdefault}{\sfdefault}{bx}{n}

\DeclareMathOperator*{\argmin}{arg\,min}

\usepackage[utf8]{inputenc} 
\usepackage[T1]{fontenc}    
\usepackage{url}            
\usepackage{booktabs}       
\usepackage{amsfonts}       
\usepackage{nicefrac}       
\usepackage{microtype}      
\usepackage{fullpage}

\usepackage{hyperref}
\usepackage{cleveref}

\usepackage{algorithm}
\usepackage{algorithmic}
\usepackage{amsthm} 
\usepackage{natbib, comment}

\newtheorem{theorem}{Theorem}
\newtheorem{proposition}{Proposition} 
\newtheorem{lemma}{Lemma} 
\newtheorem{assum}{Assumption}

\crefname{assumption}{assumption}{assumptions}
\theoremstyle{remark}

\allowdisplaybreaks

\usepackage{colortbl}
\usepackage{multirow}

\usepackage{threeparttable}
\usepackage[table]{xcolor}

\usepackage{hhline}
\usepackage{makecell}

\usepackage{wrapfig}
\usepackage{subfigure}
\usepackage{graphicx}
\usepackage{mathtools}

\usepackage{listings}

\title{Doubly Robust Instance-Reweighted Adversarial Training}

\author{Daouda Sow$^{1}$, Sen Lin$^1$, Zhangyang Wang$^2$, Yingbin Liang$^{1}$ \\
$^{1}$The Ohio State University, $^{2}$University of Texas at Austin \\
\texttt{sow.53@osu.edu, lin.4282@osu.edu, atlaswang@utexas.edu, liang889@osu.edu} \\ 
}

\allowdisplaybreaks[4]

\begin{document}

\maketitle

\begin{abstract}
Assigning importance weights to adversarial data has achieved great success in training adversarially robust networks under limited model capacity. However, existing instance-reweighted adversarial training (AT) methods heavily depend on heuristics and/or geometric interpretations to determine those importance weights, making these algorithms lack rigorous theoretical justification/guarantee. Moreover, recent research has shown that adversarial training suffers from a severe non-uniform robust performance across the training distribution, e.g., data points belonging to some classes can be much more vulnerable to adversarial attacks than others. To address both issues, in this paper, we propose a novel doubly-robust instance reweighted AT framework, which allows to obtain the importance weights via exploring distributionally robust optimization (DRO) techniques, and at the same time boosts the robustness on the most vulnerable examples. In particular, our importance weights are obtained by optimizing the KL-divergence regularized loss function, which allows us to devise new algorithms with a theoretical convergence guarantee. 
Experiments on standard classification datasets demonstrate that our proposed approach outperforms related state-of-the-art baseline methods in terms of average robust performance, and at the same time improves the robustness against attacks on the weakest data points. Codes will be available soon. 
\end{abstract} 

\section{Introduction}
Deep learning models are known to be vulnerable to malicious adversarial attacks \cite{nguyen2015deep}, i.e., small perturbation added to natural input data can easily fool state-of-the-art networks. Given that these deep neural networks are being heavily deployed in real-life applications, even in safety-critical applications, adversarial training (AT) \cite{madry2017towards,athalye2018obfuscated,carmon2019unlabeled} has been proposed for training networks to be robust to adversarial attacks \cite{athalye2018synthesizing,szegedy2013intriguing,goodfellow2014explaining,papernot2016limitations,nguyen2015deep,zhang2021causaladv,zhang2020attacks}. In particular, most existing defense strategies are based on the recipes similar to AT \cite{madry2017towards}, where 
the goal is to minimize the average loss of the worst-case adversarial data for the training distribution via solving a minimax optimization problem. 

Despite its success, the traditional AT method \cite{madry2017towards} has some major limitations. First, even though existing overparameterized neural networks seem to be good enough for natural data,
highly adversarial data consumes much more model capacity compared to their clean counterpart, making the minimization of the {\bf uniform} average adversarial loss a very pessimistic goal, as argued in \cite{zhang2020geometry}. 
To overcome this limitation, recent works \cite{zhang2020geometry,liu2021probabilistic,zeng2021adversarial,ding2018mma} assign an {\bf importance weight} to each data point in the training distribution, in order to emphasize the ones that are critical to determining the model's decision boundaries. By allowing more careful exploitation of the limited model capacity, such a simple {\bf instance-reweighted} scheme combined with traditional adversarial training has yielded a significant boost in the robust performance of current adversarially trained models. Yet, existing methods for instance-reweighted AT mostly adopt heuristic techniques and/or geometric intuitions in order to compute the instance weights, which makes these algorithms lack a principled and rigorous theoretical justification/guarantee. This hence motivates the following question we ask:

\textit{How to systematically determine the importance weights via a principled approach, rather than resorting to heuristics/interpretations which are often \textit{sub-optimal}?} 


\noindent Moreover, as observed in \cite{tian2021analysis}, another critical limitation of the transitional AT method is that it suffers a severe {\em non-uniform} performance across the empirical distribution. For example, while the average robust performance of the AT method on the CIFAR10 dataset can be as high as $49\%$, the robust accuracy for the weakest class is as low as $14\%$, which depicts a huge disparity in robust performance across different classes. We note that such a non-uniform performance across classes is also slightly observed in the standard training with clean data, but its severity is much worsened in adversarial training (see \Cref{fig:dist}). Indeed, this is a critical limitation that requires special attention as, in a real-world situation, a more intelligent attacker can, in fact, decide which examples to attack so as to achieve a much higher success rate (e.g., $87\%$ when attacking the most vulnerable class). 
This non-uniform robust performance is even worsened in the case of imbalanced training distributions \cite{wu2021adversarial,wang2022imbalanced}, where the robust performance for the most vulnerable class can be as low as $0\%$. 
This motivates our second question given below:

\textit{Can such an issue of non-uniform performance particularly over imbalanced datasets be addressed at the instance level simultaneously as we design the importance weights to address the first question?}

\noindent In this paper, we propose a novel doubly robust instance reweighted optimization approach to address both of the above questions. 
\subsection{Our Contributions}
\textbf{(A novel principled framework for instance reweighted AT)}~
In order to determine the instance weights for AT in a theoretically grounded way, we propose a novel doubly robust instance reweighted optimization framework,
based on  distributionally robust optimization (DRO) \cite{rahimian2019distributionally,qian2019robust} and bilevel optimization  \citep{zhang2022revisiting,pedregosa2016hyperparameter,grazzi2020iteration}. 
Through building a model that is robust not only to the adversarial attacks but also to the worst-case instance weight selections, our framework (a) enjoys better robust performance than existing instance-reweighted schemes based on heuristic/geometric techniques \cite{zhang2020geometry,liu2021probabilistic,zeng2021adversarial} as well as tradtional AT baselines \cite{madry2017towards}; and (b) addresses the non-uniform issues \cite{tian2021analysis,pethick2023revisiting} of traditional AT by carefully optimizing 
the instance weights so as to boost the robust performance of the most vulnerable examples. Moreover, the proposed framework can be reformulated into a new finite-sum compositional bilevel optimization problem (CBO), which can be of great interest to the optimization community on its own.


\noindent \textbf{(A novel algorithm with theoretical guarantee)}~ 
Solving the proposed doubly robust optimization problem is technically challenging, including the non-differentiability of the optimizer for the constrained inner level problem and the biased hypergradient estimation for the compositional outer level problem. To tackle these challenges, we first propose a penalized reformulation based on the log-barrier penalty method, and then develop a novel algorithm which exploits the implicit function theorem and keeps track of a running average of the outer level composed function values. Our algorithm not only leads to a robust model for the proposed instance reweighted optimization problem but also provides a solution to the generic compositional bilevel optimization problem. Under widely adopted assumptions in the bilevel \citep{grazzi2020bo,ji2021bo,rajeswaran2019meta,ji2021lower} and compositional optimization \cite{wang2017stochastic,chen2021solving,lian2017finite,blanchet2017unbiased,devraj2019stochastic} literature, we further establish the convergence guarantee for the proposed algorithm.


\noindent \textbf{(Strong experimental performance)} Experiments on several balanced and imbalanced image recognition datasets demonstrate the effectiveness of our proposed approach. In particular, on CIFAR10 our approach yields +3.5\% improvement in overall robustness against PGD attacks \cite{madry2017towards} with most of it coming from boosting robustness on vulnerable data points. 

\subsection{Related Work}
\textbf{Adversarial training for robust learning.} 
Adversarial training (AT) \cite{madry2017towards,athalye2018obfuscated,carmon2019unlabeled} was proposed for training deep neural networks robust to malicious adversarial attacks \cite{goodfellow2014explaining,tramer2017space}. In particular, \cite{madry2017towards} introduced a generic AT framework based on minimax optimization with the goal of minimizing the training loss of the worst-case adversarial data for the training distribution. 
However, despite AT method being still considered as one of the most powerful defense strategies, \cite{rice2020overfitting} highlights a severe decrease in robust performance of traditional AT when training is not stopped early, a phenomenon they dubbed \textit{robust overfitting}. Several extensions of the standard AT method have been proposed to mitigate this intriguing problem, such as data augmentation-based techniques \cite{rebuffi2021data,gowal2021improving}, or smoothing-based methods \cite{chen2021robust,yang2020adversarial,yang2020closer}. \cite{zhang2019theoretically} proposed a theoretically grounded objective for AT to strike a balance between robust and natural performance. 
However, those methods suffer a severe non-uniform performance across classification categories, as observed in \cite{tian2021analysis}. Our proposed framework helps mitigate this drawback by carefully optimizing for the most vulnerable data points. 



\noindent \textbf{Instance reweighted adversarial training.}
Another line of works \cite{zhang2020geometry, liu2021probabilistic,zeng2021adversarial,ding2018mma} assign an importance weight to each data point in the empirical distribution and minimize the weighted adversarial losses. This has been shown to significantly boost the performance of AT due to more careful exploitation of the limited capacity of large deep neural networks to fit highly adversarial data, and helps overcome robust overfitting to some extent \cite{zhang2020geometry}. 
For example, in the geometry-aware adversarial instance reweighted adversarial training (GAIRAT) \cite{zhang2020geometry} method, the instance weight is computed based on the minimum number of PGD \cite{madry2017towards} steps required to generate a mis-classified adversarial example. \cite{liu2021probabilistic} leverages probabilistic margins to compute weights. Existing approaches for instance reweighted AT are, however, all based on heuristics/geometric intuitions to determine the weights. In this paper, we propose a principled approach to instance-reweighted AT by exploiting robust optimization techniques \cite{qian2019robust,rahimian2019distributionally}. 

\noindent Instance reweighting has also been used in the context of domain adaptation \cite{jiang2007instance}, data augmentation \cite{yi2021reweighting}, and imbalanced classification \cite{ren2018learning}. By determining the instance weights in a more principled way, our method also has the potential to be applied to these contexts, which we leave as future work.

\noindent\textbf{Bilevel optimization.} Bilevel optimization is a powerful tool to study many machine learning applications such as hyperparameter optimization \citep{franceschi2018bilevel,shaban2019truncated}, meta-learning \citep{bertinetto2018meta,franceschi2018bilevel,rajeswaran2019meta,ji2020convergence,liu2021investigating}, neural architecture search \citep{liu2018darts,zhang2021idarts}, etc. Existing approaches are usually approximate implicit differentiation (AID) based \citep{domke2012opt,pedregosa2016hyperparameter,gould2016differentiating,liao2018opt,lorraine2020ho}, or iterative differentiation (ITD) based \citep{domke2012opt,maclaurin2015gradient,franceschi2017forward,finn2017model,shaban2019truncated,rajeswaran2019meta,liu2020generic}. The convergence rates of these methods have been widely established \citep{grazzi2020bo,ji2021bo,rajeswaran2019meta,ji2021lower}.
Bilevel optimization has been leveraged in adversarial training very recently, which provides a more generic framework by allowing independent designs of the inner and outer level objectives \cite{zhang2022revisiting}.
However, none of these studies investigated bilevel optimization when the outer objective is in the form of compositions of functions. In this work, we introduce the compositional bilevel optimization problem as a novel pipeline for instance reweighted AT, and establish its first known convergence rate. 


\noindent\textbf{Stochastic compositional optimization.}
Stochastic compositional optimization (SCO) deals with the minimization of compositions of stochastic functions. \cite{wang2017stochastic} proposed the compositional stochastic gradient descent (SCGD) algorithm as a pioneering method for SCO problems and established its convergence rate. Many extentions of SCGD have been proposed with improved rates, including accelerated and adaptive SCGD methods \cite{wang2016accelerating,tutunov2020compositional}, and variance reduced SCGD methods \cite{lian2017finite,blanchet2017unbiased,lin2020improved,devraj2019stochastic,hu2019efficient}. A SCO reformulation has also been used to solve nonconvex distributionally robust optimization (DRO) \cite{rahimian2019distributionally,qian2019robust} problems. The problem studied in this paper naturally falls into a new class of problems but with an additional inner loop compared to the existing single-level SCO problem, which we refer to as compositional bilevel optimization (CBO). 

\section{Preliminary on AT} 
\textbf{Traditional AT.}
The traditional adversarial training (AT) \cite{madry2017towards} framework is formulated as the following minimax optimization problem over the training dataset $\mathcal{D}=\{(x_i, y_i)\}_{i=1}^M$ 
{\small
\begin{align} \label{eq:at}
    \underset{\theta}\min ~ \frac{1}{M} \sum_{i=1}^M \max_{\delta \in \mathcal{C}} \ell(x_i+\delta, y_i; \theta), 
\end{align} }%
where $\ell(x_i+\delta, y_i; \theta)$ is the loss function on the adversarial input $x_i + \delta$, $\mathcal{C}$ is the treat model that defines the constraint on the adversarial noise $\delta$, and $\theta \in \mathbb{R}^d$ corresponds to the model parameters. Thus, the traditional AT builds robust models by optimizing the parameters $\theta$ for the average worst-case adversarial loss $\ell(x_i+\delta, y_i; \theta)$ over the training dataset $\mathcal{D}$. A natural solver for the problem in \Cref{eq:at} is the AT algorithm \cite{madry2017towards}, where 1) the projected gradient descent (PGD) \cite{madry2017towards} method is first adopted to approximate the worst-case adversarial noise $\delta$ and 2)  an outer minimization step is performed on the parameters $\theta$ using stochastic gradient descent (SGD) methods. However, the traditional AT is known to consume tremendous amount of model capacity due to its overwhelming smoothing effect of natural data neighborhoods \cite{zhang2020geometry}. In other words, the traditional AT robustifies models by making decision boundaries far away from natural data points so that their adversarial counterparts are still correctly classified (i.e., do not cross the decision boundary), and thus requires significantly more model capacity compared to the standard training on clean data. 




\noindent \textbf{Instance Reweighted AT.} 
The geometry-aware approach in \cite{zhang2020geometry} introduces a new line of methods that {\bf reweights} the adversarial loss on each individual data point in order to address the drawback of traditional AT. The key motivation is that distinct data points are unequal by nature and should be treated differently based on how important they participate on the selection of decision boundaries. Hence, the learning objective of the geometry-aware instance-reweighted adversarial training (GAIRAT) method as well as its variants \cite{zhang2020geometry,liu2021probabilistic,zeng2021adversarial} can be written as 
{\small
\begin{align} \label{eq:wat} 
  \underset{\theta}\min ~ \sum_{i=1}^M w_i \max_{\delta \in \mathcal{C}_i} \ell(x_i+\delta, y_i; \theta) \quad \textrm{with} \quad \sum_{i=1}^M w_i = 1 \textrm{  and  } w_i \geq 0, 
\end{align}}%
where the constraints on the weights vector $w =(w_1, ..., w_M)^\top$ are imposed in order to make \Cref{eq:wat} consistent with the original objective in \Cref{eq:at}.  
This framework assumes that the weight vector $w =(w_1, ..., w_M)^\top$ can be obtained separately and the goal is only to optimize for $\theta$ once an off-the-shelf technique/heuristic can be used to compute $w$. 
Intuitively, the key idea driving the weight assignments in instance reweighted methods is that
larger weights should be assigned to the
training examples closer to the decision boundaries, whereas the ones that are far away should have smaller weights because they are less important in determining the boundaries. The major difference among the existing instance reweighted AT methods lies in the heuristics used to design/compute the instance weights $w_i, i=1, ..., M$. However, none of those methods adopt a scheme that is theoretically grounded, nor does the formulation in \Cref{eq:wat} provide a way of determining those weights. 

\noindent\textbf{Bilevel Optimization Formulation for AT.} 
Along a different line, bilevel optimization has recently been leveraged to develop a more powerful framework for adversarial training \cite{zhang2022revisiting}:
{\small
\begin{align} \label{eq:atblo}
    \underset{\theta}\min ~ \frac{1}{M} \sum_{i=1}^M \ell(x_i+\delta^*_i(\theta), y_i; \theta) \quad \textrm{s.t.} \quad  \delta^*_i(\theta) =\argmin_{\delta \in \mathcal{C}_i} \ell'(x_i+\delta,y_i; \theta),
\end{align}}
where for each data point $(x_i, y_i)$, $\delta^*_i(\theta)$ represents some worst-case/optimal adversarial noise under the attack loss function $\ell'(\cdot; \theta)$. 
Such a bilevel optimization formulation of AT has key advantages over the traditional framework in \Cref{eq:at}. First,  the traditional AT can be recovered by setting the attack objective to be the negative of the training objective, i.e., $\ell'(\cdot; \theta) = -\ell(\cdot; \theta)$. Second, the bilevel formulation gives one the flexibility to separately design the inner and outer level objectives, $\ell'$ and $\ell$, respectively. These key advantages make the formulation in \Cref{eq:atblo} a more generic and powerful framework than the one in \Cref{eq:at}. As we will see next, this enables us to independently construct a new outer level objective that also solves for the instance weights $w$, and an inner level objective for regularized attack. 
\section{Proposed Framework for Instance Reweighted AT}
\subsection{DONE: \underline{D}oubly R\underline{o}bust I\underline{n}stance R\underline{e}weighted AT} 


Using the bilevel formulation for AT in Eq. \eqref{eq:atblo}, we can incorporate the instance reweighted idea as 
{\small
\begin{align} \label{eq:watblo}
    \underset{\theta}\min ~ \sum_{i=1}^M w_i \ell(x_i+\delta^*_i(\theta), y_i; \theta) \textrm{ s.t. } \delta^*_i(\theta) =\argmin_{\delta \in \mathcal{C}_i} \ell'(x_i+\delta,y_i; \theta) \textrm{ with } \sum_{i=1}^M w_i = 1 \textrm{ and } w_i \geq 0. 
\end{align} }%

Based on bilevel optimization and distributionally robust optimization (DRO), we next propose a new framework for AT which determines the weights $w$ in  a more principled way rather than using heuristic methods. Specifically, by letting $w$ maximize the weighted sum of the adversarial losses $\ell (x_i+\delta^*_i(\theta), y_i; \theta), i=1, ..., M$, we seek to build a model in the outer level problem that is robust not only to the adversarial attacks but also to the worst-case attack distribution:
{\small
\begin{align} \label{eq:droblo1}
    \underset{\theta}\min ~ { \underset{w \in \mathcal{P}}{\max}}  ~ \sum_{i=1}^M w_i \ell(x_i+\delta^*_i(\theta), y_i; \theta) { - r\sum_{i=1}^M w_i\log(Mw_i)} \quad \textrm{s.t.} \quad \delta^*_i(\theta) =\argmin_{\delta \in \mathcal{C}_i} \ell'(x_i+\delta,y_i; \theta), 
\end{align}}%
where $\mathcal{P}$ represents the probability simplex, i.e., $\mathcal{P} = \{w \in \mathbb{R}^M: \sum_{i=1}^M w_i = 1 \textrm{  and  } w_i \geq 0\}$, 
and the term $r\sum_{i=1}^M w_i\log(Mw_i)$ in the outer level objective captures the KL-divergence between $w$ and the uniform weight distribution, which is a widely adopted choice of regularizer in the DRO literature \cite{rahimian2019distributionally}. Note that the regularization parameter $r>0$ controls the tradeoff between two extreme cases: 1) $r=0$ leads to an un-regularized problem (as we comment below), and 2) $r \rightarrow \infty$ yields $w_i \rightarrow \frac{1}{M}$, and hence, we recover the average objective in \Cref{eq:at}. Such a regularizer is introduced to promote the balance between the uniform and worst-case weights $w$; otherwise the outer level objective in \Cref{eq:droblo1} becomes linear in weights vector $w$, which makes the solution of the `$\max$' problem to be trivially a one-hot vector $w$ (where the only `1' is at index $i$ with the largest adversarial loss), and in practice, such a trivial one-hot vector $w$ makes the optimization routine unstable and usually hurts generalization to the training distribution \cite{qian2019robust,wang2021adversarial}. 

Overall, the formulation in \Cref{eq:droblo1} becomes a {\bf doubly robust} bilevel optimization: (a) the inner level finds the worst-case noise $\delta$ in order to make the model parameters $\theta$ robust to such adversarial perturbation of data input; and (b) the outer level finds the worst-case reweighting first so that the optimization over the model $\theta$ can focus on those data points with high loss values, i.e., the optimization over $\theta$ is over the worst-case  adversarial losses.


\subsection{An Equivalent Compositional Bilevel Optimization Problem}

An important consequence of choosing the KL-divergence as the regularizer is that the $\max$ problem in the outer objective of \Cref{eq:droblo1} admits a unique solution $w^*(\theta)$ (see \cite{qi2021online} for proof), which has its $i$-the entry given by $w^*_i(\theta) = \exp\left(\frac{\ell_i(\theta, \delta^*_i(\theta))}{r}\right) / \sum_j \exp\left(\frac{\ell_j(\theta, \delta^*_j(\theta))}{r}\right)$. Here we denote $\ell_i(\theta, \delta^*_i(\theta)) = \ell(x_i+\delta^*_i(\theta), y_i; \theta)$. Substituting this optimal weights vector $w^*(\theta)$ back in \Cref{eq:droblo1} yields the following equivalent optimization problem 
{\small
\begin{align} \label{eq:cbloat}
    &\underset{\theta}\min ~ r\log \left(\frac{1}{M} \sum_{i=1}^M \exp\left(\frac{\ell_i(\theta, \delta^*_i(\theta))}{r}\right) \right) \quad \textrm{s.t.} \quad \delta^*_i(\theta) =\argmin_{\delta \in \mathcal{C}_i} \ell_i'(\theta, \delta). 
\end{align}}%
Problem (\ref{eq:cbloat}) is, in fact to the best of our knowledge, a novel optimization framework, 
which we define as a compositional bilevel optimization problem. Without the inner level problem, stochastic algorithms with known convergence behaviors have been devised for the single-level compositional problem. 
Nevertheless, directly solving problem (\ref{eq:cbloat}) suffers from several key technical challenges. In particular, the fact that the minimizer of the inner level constrained problem in \Cref{eq:cbloat} may not be differentiable w.r.t. to the model parameter $\theta$ prevents the usage of implicit differentiation for solving the bilevel optimization problem. 

To tackle this challenge, we propose a penalized reformulation based on the log-barrier penalty method.
More specifically, we consider  $\ell_\infty$-norm based attack constraint given by $\mathcal{C} = \{ \delta \in \mathbb{R}^p: \big\|\delta\big\|_\infty \leq \epsilon, x+\delta \in [0,1]^p\}$ for radius $\epsilon>0$ and input $x \in \mathbb{R}^p$. In this case, the constraint set $\mathcal{C}$ can be written in the form of linear constraint $A\delta \leq b$ with $A = \big(I_p, -I_p\big)^\top \in \mathbb{R}^{2p \times p}$ and $b=\big(\min (\epsilon \mathbf{1}_p, \mathbf{1}_p-x), \min (\epsilon \mathbf{1}_p, x)\big)^\top \in \mathbb{R}^{2p}$. With this, we can reformulate the inner problem in \Cref{eq:cbloat} as $\delta^*_i(\theta) =\argmin_{\{A_i \delta \leq b_i\}} \ell_i'(\theta, \delta)$, where $A_i$ and $b_i$ are realizations of aforementioned $A$ and $b$ for input $x_i$. By using the log-barrier penalty method to penalize the linear constraint into the attack objective, the optimization problem (\ref{eq:cbloat}) becomes 
{\small
\begin{align}\label{eq:bar}
    \underset{\theta}\min ~ \mathcal{L}(\theta) \coloneqq r\log \left(\frac{1}{M} \sum_{i=1}^M \exp\left(\frac{\ell_i(\theta, \hat \delta^*_i(\theta))}{r}\right) \right) \text { s.t. } \hat \delta^*_i(\theta) = \argmin_{\delta \in \mathcal{C}_i} \ell_i^{bar}(\theta, \delta),
\end{align}}%
where $\ell_i^{bar}(\theta, \delta) \coloneqq \ell_i'(\theta, \delta) - c \sum_{k=1}^{2p} \log (b_k - \delta^\top a_k)$, $a_k$ denotes the $k$-th row of matrix $A_i$ and $b_k$ is the $k$-th entry of vector $b_i$. Note that now the constraint $\{\delta \in \mathcal{C}_i\}$ is never binding in \Cref{eq:bar}, because the log-barrier penalty forces the minimizer of $\ell_i^{bar}(\theta, \delta)$ to be strictly inside the constraint set. Based on this, we show that the minimizer $\hat \delta^*_i(\theta)$ becomes differentiable, i.e., $\frac{\partial \hat \delta^*_i(\theta)}{\partial \theta}$ exists when $\ell_i'(\theta, \delta)$ is twice differentiable and under some mild conditions. With the smoothness of $\hat \delta^*_i(\theta)$, we also provide the expression of the gradient $\nabla \mathcal{L}(\theta)$ in the following proposition.  
\begin{proposition}\label{prop:igrad}
Let $\ell_i'(\theta, \delta)$ be twice differentiable. Define $\gamma_k = 1 / (b_k -a_k^\top \hat \delta^*_i(\theta))^2$, $k=1,..., 2p$ and diagonal matrix $C_i(\theta) = c \operatorname{diag}\big(\gamma_1 + \gamma_{p+1}, \gamma_2 + \gamma_{p+2}, ..., \gamma_p + \gamma_{2p} \big)$. 
If $\nabla_\delta^2 ~\ell_i'(\theta, \hat \delta^*_i(\theta)) + C_i(\theta)$ is invertible, then the implicit gradient $\frac{\partial \hat \delta^*_i(\theta)}{\partial \theta}$ exists and we have 
{\small
\begin{align}
    \nabla \mathcal{L}(\theta) = \frac{r \sum_{i=1}^M \Big( \nabla_{\theta}~g_i(\theta, \hat \delta^*_i(\theta)) - \nabla_{\theta \delta}~\ell_i'(\theta, \hat \delta^*_i(\theta)) \big[\nabla_\delta^2 ~\ell_i'(\theta, \hat \delta^*_i(\theta)) + C_i(\theta)\big]^{-1} \nabla_{\delta}~g_i(\theta, \hat \delta^*_i(\theta)) \Big) }{\sum_{i=1}^M g_i(\theta, \hat \delta^*_i(\theta))}, \nonumber 
\end{align}}%
where $g_i(\theta, \hat \delta^*_i(\theta)) = \exp\left(\frac{\ell_i(\theta, \hat \delta^*_i(\theta))}{r}\right)$. 
\end{proposition} 
\Cref{prop:igrad} provides the expression of the total gradient $\nabla \mathcal{L}(\theta)$, which is useful for practical implementation of implicit differentiation based algorithms for problem (\ref{eq:cbloat}). Moreover, as in \cite{zhang2022revisiting}, when $\ell_i'(\theta, \cdot)$ is  modeled by a ReLU-based deep neural network, the hessian $\nabla_\delta^2 ~\ell_i'(\theta, \delta)$ w.r.t. input $\delta$ can be safely neglected due to the fact that ReLU network generally lead to piece-wise linear decision boundaries w.r.t. its inputs \cite{moosavi2019robustness,alfarra2022decision}, i.e., $\nabla_\delta^2 ~\ell_i'(\theta, \delta) \approx 0$. Further, the diagonal matrix $C_i(\theta)$ can be efficiently inverted. Hence, in order to approximate $\nabla \mathcal{L}(\theta)$, we only need Jacobian-vector product computations which can be efficiently computed using existing automatic differentiation packages.

\subsection{Compositional Implicit Differentiation (CID)}


To solve our reformulated problem (\ref{eq:bar}) for AT, we consider the following generic compositional bilevel optimization problem, which can be of great interest to the optimization community: 
{\small
\begin{align} \label{eq:cblo} 
&\underset{\theta}\min \hspace{2pt} F(\theta):= f\left(g \left(\theta, \delta^{*}(\theta) \right)\right) = f\left(\frac{1}{M} \sum_{i=1}^{M} g_i \left(\theta, \delta_i^{*}(\theta) \right)\right) \\ 
&\text {s.t. } \delta^{*}(\theta) = \left(\delta^{*}_1(\theta), ..., \delta^{*}_M(\theta)\right) = \underset{\left(\delta_1, ..., \delta_M\right) \in \mathcal{V}_1 \times ...\times \mathcal{V}_M}{\argmin} \frac{1}{M} \sum_{i=1}^{M} h_i\left(\theta, \delta_i \right), \nonumber
\end{align}}%
which can immediately recover problem (\ref{eq:bar}) by setting $g_i = \exp\left(\frac{\ell_i(\theta, \hat \delta^*_i(\theta))}{r}\right)$, $h_i = \ell_i'(\theta, \delta) - c \sum_{k=1}^{2p} \log (b_k - \delta^\top a_k)$, and the constraint set $\mathcal{V}_i=\mathcal{C}_i$. 
Here the outer functions $g_i(\theta, \delta): \mathbb{R}^d \times \mathbb{R}^p \rightarrow \mathbb{R}^m$ and $f(z): \mathbb{R}^m \rightarrow \mathbb{R}$ are generic nonconvex and continuously differentiable functions. The inner function $h_i(\theta, \delta): \mathbb{R}^d \times \mathcal{V}_i \rightarrow \mathbb{R}$ is a twice differentiable and admits a unique minimizer in $\delta$, $\mathcal{V}_i$ is a convex subset of $\mathbb{R}^p$ that is assumed to contain the minizers $\delta^*_i(\theta)$. 
We collect all inner loop minimizers into a single vector $\delta^{*}(\theta)$. 
The goal is to minimize the total objective function $F(\theta): \mathbb{R}^d \longrightarrow \mathbb{R}$, \emph{which not only leads to a robust model for our instance reweighted optimization problem (\ref{eq:bar}) but also provides a solution to the generic compositional bilevel optimization problem}.

As alluded earlier, solving the compositional bilevel optimization problem is nontrivial. More specifically, it can be shown that
 the gradient of the total objective is $\nabla F(\theta) = \frac{\partial g \left(\theta, \delta^{*}(\theta) \right)}{\partial \theta} \nabla f\left(g \left(\theta, \delta^{*}(\theta) \right)\right)$ by applying the chain rule. Due to the fact that $\nabla f(\cdot)$ needs to be evaluated at the full value $g \left(\theta, \delta^{*}(\theta) \right)$, standard stochastic gradient descent methods cannot be naively applied here. The reason is that even if we can obtain the unbiased estimates $g_i \left(\theta, \delta^{*}_i(\theta) \right)$, the product 
$\frac{\partial g_i \left(\theta, \delta^{*}_i(\theta) \right)}{\partial \theta} \nabla f\left(g_i \left(\theta, \delta^{*}_i(\theta) \right)\right)$ would still be biased, unless $f(\cdot)$ is a linear function. This key difference makes problem (\ref{eq:cblo}) particularly challenging and sets it apart from the standard finite-sum bilevel optimization problem in which the total objective is linear w.r.t. the sampling probabilities $\frac{1}{M}$.

To design a theoretically grounded algorithm for problem (\ref{eq:cblo}), note that the stochastic compositional gradient descent (SCGD) \cite{wang2017stochastic} algorithm for the single-level compositional optimization problem  keeps track of a running average of the composed function evaluations during the algorithm running. Inspired by SCGD, we propose a novel algorithm (see \Cref{alg:cid}) that exploits the implicit differentiation technique to deal with the bilevel aspect of problem (\ref{eq:cblo}). 
Using the implicit function theorem, we can obtain 
{\small
\begin{align}\label{eq:ift}
    \frac{\partial g_i\left(\theta, \delta^{*}_i(\theta) \right)}{\partial \theta} = \nabla_{\theta} g_i\left(\theta, \delta^{*}_i(\theta) \right) - \nabla_{\theta} \nabla_{\delta} h_i\left(\theta, \delta^{*}_i(\theta) \right) v_i^*, 
\end{align}}%
with each $v_i^*$ being the solution of the linear system $ \nabla_{\delta}^2 h_i\left(\theta, \delta^{*}_i(\theta) \right) v = \nabla_{\delta} g_i\left(\theta, \delta^{*}_i(\theta) \right)$. 

Specifically, at each step $t$, the algorithm first samples a batch $\mathcal{B}$ of cost functions $\{(g_i, h_i)\}$ and applies $K$ steps of projected gradient descent to obtain $\delta^K_{i}(\theta_t)$ as an estimate of the minimizer $\delta^*_{i}(\theta_t)$ of each $h_i(\theta_t, \cdot)$ in $\mathcal{B}$. Then, the algorithm computes 
an approximation $\widehat{\nabla} g_i(\theta_t, \delta_i^K(\theta_t))$ of the stochastic gradient sample $\frac{\partial g_i\left(\theta_t, \delta^{*}(\theta_t) \right)}{\partial \theta}$ by replacing each $\delta^*_{i}(\theta_t)$ with $\delta^K_{i}(\theta_t)$ in \Cref{eq:ift}. The running estimate $u_t$ of $\frac{\partial g\left(\theta, \delta^{*}(\theta) \right)}{\partial \theta}$ and the parameters $\theta$ will be next updated as follows
{\small
\begin{align}\label{eq:update}
    u_{t+1} = (1-\eta_t)u_{t} + \frac{\eta_t}{|\mathcal{B}|}\sum_{i=1}^{|\mathcal{B}|} g_i(\theta_t, \delta_i^K(\theta_t)) \textrm{ and }
    \theta_{t+1} = \theta_{t} - \frac{\beta_t}{|\mathcal{B}|}\sum_{i=1}^{|\mathcal{B}|} \widehat{\nabla} g_i(\theta_t, \delta_i^K(\theta_t)) \nabla f(u_{t+1}). 
\end{align}
}%
Note that we will refer the instantiation of \Cref{alg:cid} for solving the instance reweighted problem (\ref{eq:bar}) as DONE (which stands for \underline{D}oubly R\underline{o}bust I\underline{n}stance R\underline{e}weighted AT). 

\begin{algorithm}[t]
	\caption{Compositional Implicit Differentiation (CID)}
	\small
	\label{alg:cid}
	\begin{algorithmic}[1]
		\STATE {\bfseries Input:} stepsizes $\alpha, \{\beta_t\}, \{\eta_t\}$, initializations $\theta_0 \in \mathbb{R}^d$, $\delta^0 \in \mathbb{R}^p$, and $u_0 \in \mathbb{R}^m$. 
		\FOR{$k=0,1,2,...,T-1$}
		\STATE{Draw a minibatch of cost functions $\mathcal{B} = \{(g_i, h_i)\}$}
        \FOR{each $(g_i, h_i) \in \mathcal{B}$ (in parallel)} 
        \FOR{$k=1,...,K$}
        \STATE{Update $\delta^{k}_{i,t} = \Pi_{\mathcal{C}} \left( \delta^{k-1}_{i,t} - \alpha\nabla_\delta h_i(\theta_t, \delta^{k-1}_{i,t})\right)$}
		\ENDFOR
        \STATE{Compute sample gradient estimate $\widehat{\nabla} g_i(\theta_t, \delta_{i,t}^K)$ as in \Cref{eq:ift} by replacing $\delta_{i}^*(\theta_t)$ with $\delta_{i,t}^K$} 
		\ENDFOR
        \STATE{Compute $g(\theta_t, \delta_{t}^K; \mathcal{B}) = \frac{1}{|\mathcal{B}|}\sum_{i=1}^{|\mathcal{B}|} g_i(\theta_t, \delta_{i,t}^K)$ and $\widehat \nabla g(\theta_t, \delta_{t}^K; \mathcal{B}) = \frac{1}{|\mathcal{B}|}\sum_{i=1}^{|\mathcal{B}|} \widehat \nabla g_i(\theta_t, \delta_{i,t}^K)$} 
        \STATE{Update $u_{t+1} = (1-\eta_t)u_{t} + \eta_t g(\theta_t, \delta_{t}^K; \mathcal{B})$}
        \STATE{Update $\theta_{t+1} = \theta_{t} - \beta_t \widehat \nabla g(\theta_t, \delta_{t}^K; \mathcal{B}) \nabla f(u_{t+1})$}
        \ENDFOR
	\end{algorithmic}
\end{algorithm}

\subsection{Convergence Analysis of CID}
In the following, we establish the convergence rate of the proposed CID algorithm under widely adopted assumptions in bilevel and compositional optimization literatures (see \Cref{app:theory} for the statement of assumptions). 
\begin{theorem} \label{thr}
Suppose that Assumptions \ref{ass:lip}, \ref{ass:h}, \ref{ass:bdv} (which are given in Appendix) hold. Select the stepsizes as $\beta_t =\frac{1}{\sqrt{T}}$ and $\eta_t \in [\frac{1}{2}, 1)$, and batchsize as $\mathcal{O} (T)$. Then, the iterates $\theta_t, t=0,..., T-1$ of the CID algorithm satisfy 
\begin{align}
    \frac{\sum_{t=0}^{T-1} \mathbb{E} \big\| \nabla F(\theta_t) \big\|^2}{T} \leq 
    \mathcal{O} \Big( \frac{1}{\sqrt{T}} + (1-\alpha\mu)^K \Big), \nonumber 
\end{align}
\end{theorem} 
The proof can be found in \Cref{app:theory}. \Cref{thr} indicates that \Cref{alg:cid} can achieve an $\epsilon$-accurate stationary point by selecting $T=\mathcal{O}(\epsilon^{-2})$ and $K = \mathcal{O}(\log \frac{1}{\epsilon})$. 
The dependency on the batchsize can be reduced to $\mathcal{O}(\epsilon^{-1})$ by selecting $\eta_t = T^{-0.25}$, which would also lead to a higher iteration complexity of $\mathcal{O}(\epsilon^{-4})$. 

\section{Experiments} 
\subsection{Experimental Setup} 
\textbf{Datasets and Baselines.} We consider image classification problems and compare the performance of our proposed DONE method with related baselines on four image recognition datasets CIFAR10 \cite{krizhevsky2009learning}, SVHN \cite{netzer2011reading}, STL10 \cite{pmlr-v15-coates11a}, and GTSRB \cite{stallkamp2012man}. More details about the datasets can be found in the appendix. We compare against standard adversarial training methods AT \cite{madry2017towards} and FAT \cite{zhang2020attacks}, and three other state-of-the-art instance re-weighted adversarial training methods GAIRAT \cite{zhang2020geometry}, WMMR \cite{zeng2021adversarial}, and MAIL \cite{liu2021probabilistic}. We use the official publicly available codes of the respective baselines and their recommended training configurations.  
For our algorithm DONE, we consider three implementations based on how we solve the inner loop optimization: (i) DONE-GD uses simple non-sign projected gradient descent steps; (ii) DONE-ADAM employs the Adam optimizer; and
(iii) DONE-PGD adopts the projected gradient sign method. 
We run all baselines on a single NVIDIA Tesla P100 GPU. 

More details about the training and hyperparameters search can be found in \Cref{app:exp}. 

\noindent \textbf{Evaluation.} For all baselines, we report their standard accuracy on clean data (SA), the robust accuracy against 20 steps PGD attacks (RA-PGD) \cite{madry2017towards}, the robust accuracy against AutoAttacks (RA-AA) \cite{croce2020reliable}, and the RA-PGD of the 30\% most vulnerable classes (RA-Tail-30) as a measure of robustness against attacks on the most vulnerable data points. 

\subsection{Better Distribution of Robust Performance} 
We first demonstrate that our proposed doubly robust formulation can indeed achieve robust performance in a more balanced way across the empirical distribution. \Cref{fig:dist} shows the per class robust accuracy (RA-PGD) of the standard AT method and our doubly-robust approach (i.e., vanilla DONE-GD method) for both balanced and imbalanced (with an imbalance ratio of 0.2) CIFAR10 dataset. For the balanced case, our algorithm improves the robustness on all classes, meanwhile with a more significant boost on the weakest classes (\textit{cat}, \textit{deer}, and \textit{bird}). On the other hand, for the imbalanced data case, the classes with more examples (last five categories) heavily dominate the robust training dynamic. 
This consequently leads to very high robustness on those classes, but nearly zero robustness on the vulnerable classes (such as \textit{cat}). However, our method can still boost the per class RA-PGD on the weak classes (+11\% on average on the 3 most vulnerable classes) and at the same time maintain a superior average RA-PGD. Overall, the results for both balanced and imbalanced settings clearly demonstrate that our doubly-robust approach can, in fact, improve worst-case robustness and hence achieve superior average robust performance. 


\begin{figure}
    \centering
    \begin{tabular}{cc}
    \includegraphics[width=5.5cm]{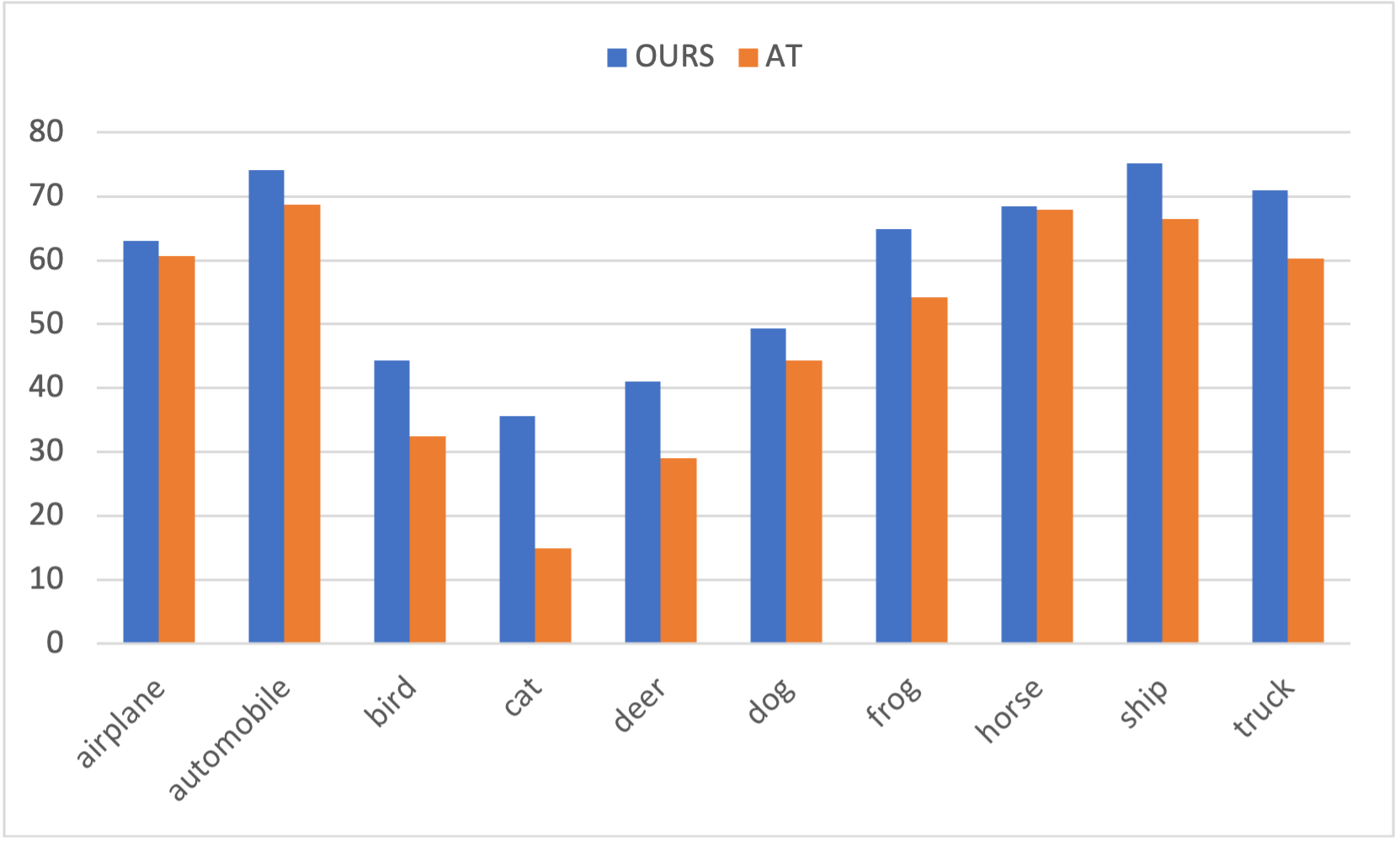}
    &\includegraphics[width=5.5cm]{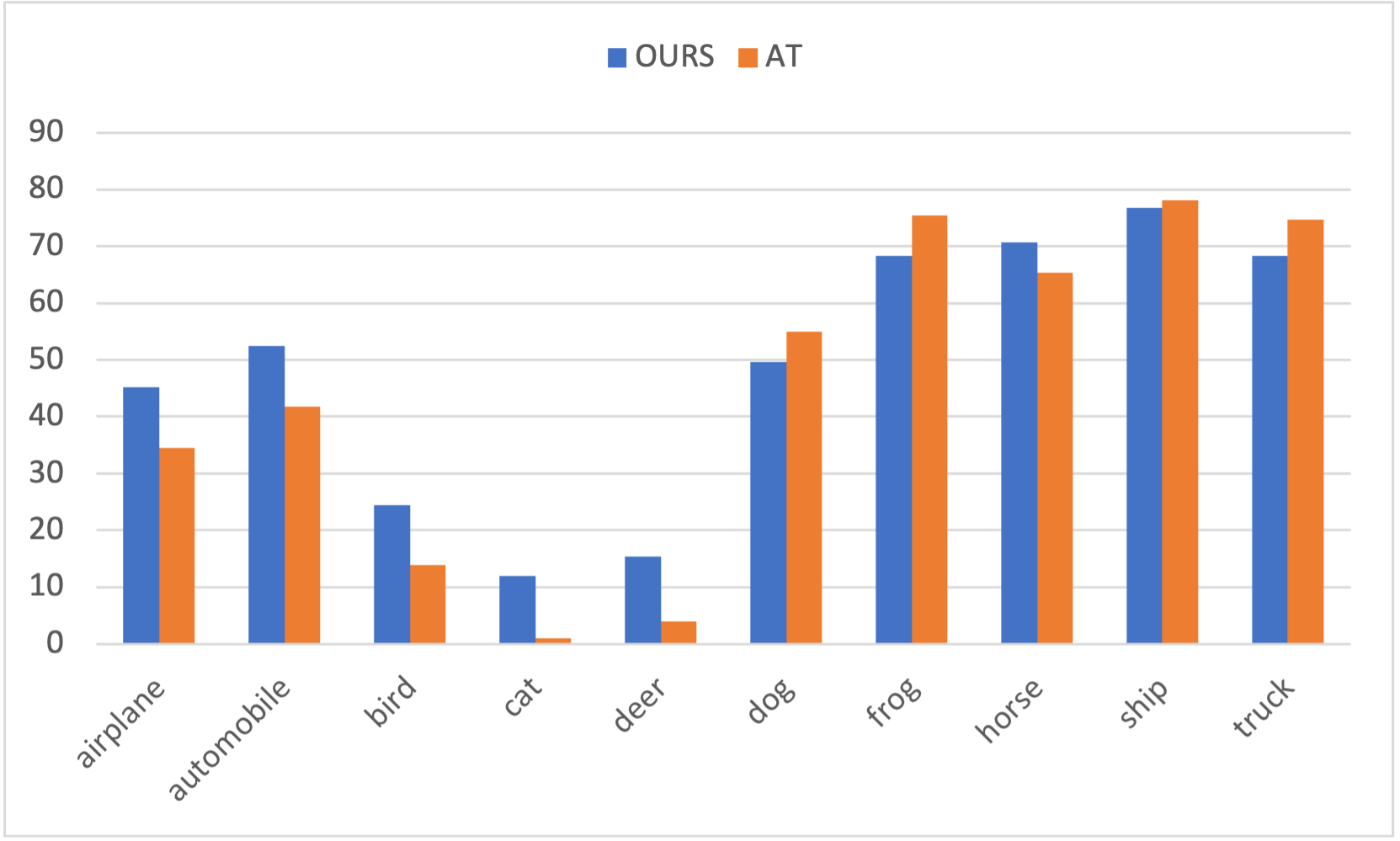}
    \end{tabular}
    \caption{Per-class robust accuracy comparisons between our method and traditional AT method on balanced and imbalanced (0.2 imbalance ratio) CIFAR10.}
    \label{fig:dist}
\end{figure}

\begin{table}[t]
\scriptsize
\centering
\caption{Performance evaluations on balanced and imbalanced (0.2 imbalance ratio) CIFAR10.}
\begin{tabular}{c|cccc|cccc}
\toprule
  \multicolumn{1}{c|}{\multirow{2}{*}{Method}}   & \multicolumn{4}{c|}{Balanced CIFAR10} & \multicolumn{4}{c}{Unbalanced CIFAR10}\\
  \cmidrule{2-9}
     & SA & RA-PGD & RA-Tail-30 & RA-AA & SA & RA-PGD & RA-Tail-30 & RA-AA\\
     \midrule
     AT & 82.1 & 49.29 & 28.35 & 45.22 & 69.74 & 42.37 & 6.25 & 39.55 \\
     FAT & \textbf{86.21} & 46.59 & 27.12 & 43.71 & - & - & - & -\\
     WMMR & 81.6 & 49.53 & 31.24 & 40.9 & - & - & - & -\\
     MAIL & 83.47 & 55.12 & 37.30 & 44.08 & 72.01 & 45.64 & 9.8 & 37.17 \\
     GAIRAT & 83.22 & 54.81 & 37.45 & 41.10 & 73.87 & 45.18 & 16.9 & 35.43 \\
     \midrule
     DONE-GD & 82.91 & 57.65 & 40.13 & \textbf{45.58} & 74.22 & \textbf{48.29} & \textbf{17.19} & \textbf{40.06} \\
     DONE-PGD & 81.68  & \textbf{58.71}  & 40.27  &  44.41 & \textbf{74.58} & 48.13 & 15.83 & 38.69 \\
     DONE-ADAM & 82.25 & 58.51 & \textbf{40.36} & 44.20 & 74.56 & 48.15 & 17.10 & 39.46 \\ 
     \bottomrule
\end{tabular}
\label{tab:cif10}
\end{table} 

\begin{table}[ht]
\scriptsize

\centering
\caption{Performance evaluations on balanced and imbalanced (0.2 imbalance ratio) SVHN.}
\begin{tabular}{c|cccc|cccc}
\toprule
  \multicolumn{1}{c|}{\multirow{2}{*}{Method}}   & \multicolumn{4}{c|}{Balanced SVHN} & \multicolumn{4}{c}{Unbalanced SVHN (0.2)}\\
  \cmidrule{2-9}
     & SA & RA-PGD & RA-Tail-30 & RA-AA & SA & RA-PGD & RA-Tail-30 & RA-AA\\
     \midrule
     AT & \textbf{93.21} & 57.82 & 47.21 & 46.27 & 88.46 & 51.08 & 33.67 & 41.13 \\
     MAIL & 93.11 & 65.56 & 52.23 & 41.38 & 86.62 & 48.48 & 31.91 & 34.46\\
     GAIRAT & 91.56 & 64.74 & 52.15 & 39.41 & 86.73 & 53.79 & 36.46 & 33.25\\
     \midrule
     DONE-PGD & 92.80 & \textbf{66.20} & \textbf{55.84} & 48.32  & 88.05 & 54.85 & 39.91 & 41.44\\
     DONE-ADAM & 92.58 & 65.72 & 53.79 & \textbf{49.13} & \textbf{88.98} & \textbf{55.90} & \textbf{41.10} & \textbf{42.38}\\ 
     \bottomrule
\end{tabular}
\label{tab:svhn}
\end{table}

\begin{table}[ht]
\scriptsize

\centering
\caption{Performance evaluations on STL10 and GTSRB (originally imbalanced) datasets.}
\begin{tabular}{c|cccc|cccc}
\toprule
  \multicolumn{1}{c|}{\multirow{2}{*}{Method}}   & \multicolumn{4}{c|}{STL10} & \multicolumn{4}{c}{GTSRB}\\
  \cmidrule{2-9}
     & SA & RA-PGD & RA-Tail-30 & RA-AA & SA & RA-PGD & RA-Tail-30 & RA-AA\\
     \midrule
     AT & 67.11 & 36.28 & 10.07 & 32.58 & 88.13 & 59.65 & 27.03 & \textbf{57.83} \\
     MAIL & \textbf{68.06} & 38.20 & 13.33 & 32.86  & 88.47 & 55.96 & 20.73 & 53.44 \\
     GAIRAT & 65.67 & 35.23 & 15.21 & 30.42  & 86.67 & 54.38 & 22.10 & 51.18 \\
     \midrule
     DONE-PGD & 66.98 & \textbf{40.23} & \textbf{17.87} & 33.71 & \textbf{89.34} & \textbf{60.16} & 27.41 & 57.25\\ 
     DONE-ADAM & 66.92 & 39.70 & 17.62 & \textbf{34.59} & 88.76 & 60.05 & \textbf{28.35} & 57.70 \\ 
     \bottomrule
\end{tabular}
\label{tab:other}
\end{table}

\begin{wraptable}{R}{7.4cm}
\centering
\footnotesize
\caption{Comparisons with fast AT methods.} 
\begin{tabular}{c|c|c|c}
\toprule Method & \text { SA } & \text { RA-PGD } & \text { RA-Tail-30 } \\
 \midrule \text { Fast-AT } & \textbf{82.44} & 45.37 & 23.3 \\
 \text { Fast-AT-GA } & 79.83 & 47.56 & 25.01 \\
 \text { Fast-BAT } & 79.91 & 49.13 & 26.05 \\
\midrule \text { DONE } & 79.17 & \textbf{55.17} & \textbf{37.13}  \\
\bottomrule
\end{tabular} 
\label{tab:fast} 
\end{wraptable}

\subsection{Main Results}
\textbf{Comparisons under CIFAR10.} 
The overall performance of the compared baselines under both balanced and imbalanced CIFAR10 are reported in \Cref{tab:cif10}. We highlight the following important observations. {\bf First,} overall our methods outperform all other baselines in terms of all three robustness metrics (RA-PGD, RA-Tail-30, and RA-AA), meanwhile also maintaining a competitive standard acurracy (SA). In particular, our algorithms can improve the RA-PGD of the strongest baseline (MAIL) by over 3\% with most of the gain coming from improvement on the weakest classes, as is depicted on the RA-Tail-30 column. This shows that our doubly robust approach can mitigate the weak robustness on the vulnerable data points while also keeping the robust performance on well guarded examples (i.e., easy data points) at the same level. 
{\bf Second}, note that the instance reweighted baselines consistently outperform the methods without reweighting on the RA-Tail-30 metric, which indicates that reweighting in general boosts the robustness on weak examples. This advantage is even clearer on the imbalanced data case. Yet, our algorithms still outperform the other instance reweighted methods by around 3\% in terms of RA-Tail-30 in the balanced data setup due to their doubly-robust nature, which clearly is helpful both for average and worst-case robust performance. {\bf Third}, note that the other methods that employ heuristics to compute the instance weights achieve worst RA-AA performance compared to the standard AT method. In contrast, our algorithms, which also fall in the instance reweighted paradigm, can still attain competitive performance for RA-AA compared to the standard AT method. This highlights the suboptimality of using heuristics which could be geared towards improving one metric (such as the RA-PGD) but may not be necessarily beneficial to the overall robustness of the model.

\noindent \textbf{Performance Comparisons on the other datasets.} \Cref{tab:svhn} shows the evaluations of the compared baselines on the SVHN dataset. As depicted, our algorithms (DONE-PGD and DONE-ADAM) significantly outperform the standard AT method on the RA-PGD metric and at the same time achieve better robustness against AutoAttacks (RA-AA). Compared with the instance reweighted baselines (MAIL \& GAIRAT), the advantage of our methods is even more important on the RA-AA metric (e.g., up to around $+8\%$ on RA-AA vs $+1.5\%$ on RA-PGD for the balanced data setting). We also note considerable improvements on the GTSRB and STL10 datasets in \Cref{tab:other}. 
Similarly to the CIFAR10 dataset, our approach yields an important boost on the RA-Tail-30 robustness metric compared to all other baselines and the advantage is more significant on the imbalanced data case. These results consistently demonstrate that our doubly-robust approach can indeed improve worst-case robust performance meanwhile also maintaining/improving the overall robustness. 

\noindent \textbf{Evaluations under Fast AT Setting.} We also compare our approach with fast adversarial training methods. For this setup, we generate the adversarial attacks during training with only 1 GD step after initialization with 1 PGD warm-up step \cite{zhang2022revisiting} and train all baselines for 25 epochs. We compare our method with Fast-BAT \cite{zhang2022revisiting}, Fast-AT \cite{wong2020fast}, and Fast-AT-GA \cite{andriushchenko2020understanding} on CIFAR10. The evaluations of the compared methods are reported in \Cref{tab:fast}. Our algorithm achieves a much better robust performance and at the same time keeps a competitive SA. In particular, we note a significant boost (+11\%) in RA-Tail-30, which is mainly the cause of the improvement in the overall RA-PGD.









\section{Conclusions}
In this paper, we proposed a novel doubly robust instance reweighted adversarial training framework based on DRO and bilevel optimization, which not only determines the instance weights for AT in a theoretically grounded way but also addresses the non-uniform issues of traditional AT by boosting the robust performance of the most vulnerable examples. To address the technical challenges in solving the doubly robust optimization problem, we proposed a penalized reformulation using the log-barrier penalty method, and developed a novel algorithm based on implicit function theorem and tracking a running average of the outer level function values. Our proposed framework also leads to a new finite-sum compositional bilevel optimization problem, which can be of great interest to the optimization community and  solved by our developed algorithm with theoretical guarantee. In the experiments on standard benchmarks, our doubly-robust approach (DONE) outperforms related state-of-the-art baseline approaches in average robust performance and also improves the robustness against attacks on the weakest data points.

\bibliography{iclr2021_conference}

\begin{thebibliography}{74}
\providecommand{\natexlab}[1]{#1}
\providecommand{\url}[1]{\texttt{#1}}
\expandafter\ifx\csname urlstyle\endcsname\relax
  \providecommand{\doi}[1]{doi: #1}\else
  \providecommand{\doi}{doi: \begingroup \urlstyle{rm}\Url}\fi

\bibitem[Alfarra et~al.(2022)Alfarra, Bibi, Hammoud, Gaafar, and
  Ghanem]{alfarra2022decision}
Motasem Alfarra, Adel Bibi, Hasan Hammoud, Mohamed Gaafar, and Bernard Ghanem.
\newblock On the decision boundaries of neural networks: A tropical geometry
  perspective.
\newblock \emph{IEEE Transactions on Pattern Analysis and Machine
  Intelligence}, 2022.

\bibitem[Andriushchenko \& Flammarion(2020)Andriushchenko and
  Flammarion]{andriushchenko2020understanding}
Maksym Andriushchenko and Nicolas Flammarion.
\newblock Understanding and improving fast adversarial training.
\newblock \emph{Advances in Neural Information Processing Systems},
  33:\penalty0 16048--16059, 2020.

\bibitem[Athalye et~al.(2018{\natexlab{a}})Athalye, Carlini, and
  Wagner]{athalye2018obfuscated}
Anish Athalye, Nicholas Carlini, and David Wagner.
\newblock Obfuscated gradients give a false sense of security: Circumventing
  defenses to adversarial examples.
\newblock In \emph{International conference on machine learning}, pp.\
  274--283. PMLR, 2018{\natexlab{a}}.

\bibitem[Athalye et~al.(2018{\natexlab{b}})Athalye, Engstrom, Ilyas, and
  Kwok]{athalye2018synthesizing}
Anish Athalye, Logan Engstrom, Andrew Ilyas, and Kevin Kwok.
\newblock Synthesizing robust adversarial examples.
\newblock In \emph{International conference on machine learning}, pp.\
  284--293. PMLR, 2018{\natexlab{b}}.

\bibitem[Bertinetto et~al.(2018)Bertinetto, Henriques, Torr, and
  Vedaldi]{bertinetto2018meta}
Luca Bertinetto, Joao~F Henriques, Philip Torr, and Andrea Vedaldi.
\newblock Meta-learning with differentiable closed-form solvers.
\newblock In \emph{International Conference on Learning Representations
  (ICLR)}, 2018.

\bibitem[Blanchet et~al.(2017)Blanchet, Goldfarb, Iyengar, Li, and
  Zhou]{blanchet2017unbiased}
Jose Blanchet, Donald Goldfarb, Garud Iyengar, Fengpei Li, and Chaoxu Zhou.
\newblock Unbiased simulation for optimizing stochastic function compositions.
\newblock \emph{arXiv preprint arXiv:1711.07564}, 2017.

\bibitem[Carmon et~al.(2019)Carmon, Raghunathan, Schmidt, Duchi, and
  Liang]{carmon2019unlabeled}
Yair Carmon, Aditi Raghunathan, Ludwig Schmidt, John~C Duchi, and Percy~S
  Liang.
\newblock Unlabeled data improves adversarial robustness.
\newblock \emph{Advances in neural information processing systems}, 32, 2019.

\bibitem[Chen et~al.(2021{\natexlab{a}})Chen, Zhang, Liu, Chang, and
  Wang]{chen2021robust}
Tianlong Chen, Zhenyu Zhang, Sijia Liu, Shiyu Chang, and Zhangyang Wang.
\newblock Robust overfitting may be mitigated by properly learned smoothening.
\newblock In \emph{International Conference on Learning Representations},
  2021{\natexlab{a}}.

\bibitem[Chen et~al.(2021{\natexlab{b}})Chen, Sun, and Yin]{chen2021solving}
Tianyi Chen, Yuejiao Sun, and Wotao Yin.
\newblock Solving stochastic compositional optimization is nearly as easy as
  solving stochastic optimization.
\newblock \emph{IEEE Transactions on Signal Processing}, 69:\penalty0
  4937--4948, 2021{\natexlab{b}}.

\bibitem[Coates et~al.(2011)Coates, Ng, and Lee]{pmlr-v15-coates11a}
Adam Coates, Andrew Ng, and Honglak Lee.
\newblock An analysis of single-layer networks in unsupervised feature
  learning.
\newblock In Geoffrey Gordon, David Dunson, and Miroslav Dudík (eds.),
  \emph{Proceedings of the Fourteenth International Conference on Artificial
  Intelligence and Statistics}, volume~15 of \emph{Proceedings of Machine
  Learning Research}, pp.\  215--223, Fort Lauderdale, FL, USA, 11--13 Apr
  2011. PMLR.
\newblock URL \url{https://proceedings.mlr.press/v15/coates11a.html}.

\bibitem[Croce \& Hein(2020)Croce and Hein]{croce2020reliable}
Francesco Croce and Matthias Hein.
\newblock Reliable evaluation of adversarial robustness with an ensemble of
  diverse parameter-free attacks.
\newblock In \emph{International conference on machine learning}, pp.\
  2206--2216. PMLR, 2020.

\bibitem[Devraj \& Chen(2019)Devraj and Chen]{devraj2019stochastic}
Adithya~M Devraj and Jianshu Chen.
\newblock Stochastic variance reduced primal dual algorithms for empirical
  composition optimization.
\newblock \emph{Advances in Neural Information Processing Systems}, 32, 2019.

\bibitem[Ding et~al.(2018)Ding, Sharma, Lui, and Huang]{ding2018mma}
Gavin~Weiguang Ding, Yash Sharma, Kry Yik~Chau Lui, and Ruitong Huang.
\newblock Mma training: Direct input space margin maximization through
  adversarial training.
\newblock \emph{arXiv preprint arXiv:1812.02637}, 2018.

\bibitem[Domke(2012)]{domke2012opt}
Justin Domke.
\newblock Generic methods for optimization-based modeling.
\newblock \emph{International Conference on Artificial Intelligence and
  Statistics (AISTATS)}, pp.\  318–326, 2012.

\bibitem[Finn et~al.(2017)Finn, Abbeel, and Levine]{finn2017model}
Chelsea Finn, Pieter Abbeel, and Sergey Levine.
\newblock Model-agnostic meta-learning for fast adaptation of deep networks.
\newblock In \emph{Proc. International Conference on Machine Learning (ICML)},
  pp.\  1126--1135, 2017.

\bibitem[Franceschi et~al.(2017)Franceschi, Donini, Frasconi, and
  Pontil]{franceschi2017forward}
Luca Franceschi, Michele Donini, Paolo Frasconi, and Massimiliano Pontil.
\newblock Forward and reverse gradient-based hyperparameter optimization.
\newblock In \emph{International Conference on Machine Learning (ICML)}, pp.\
  1165--1173, 2017.

\bibitem[Franceschi et~al.(2018)Franceschi, Frasconi, Salzo, Grazzi, and
  Pontil]{franceschi2018bilevel}
Luca Franceschi, Paolo Frasconi, Saverio Salzo, Riccardo Grazzi, and
  Massimiliano Pontil.
\newblock Bilevel programming for hyperparameter optimization and
  meta-learning.
\newblock In \emph{International Conference on Machine Learning (ICML)}, pp.\
  1568--1577, 2018.

\bibitem[Goodfellow et~al.(2014)Goodfellow, Shlens, and
  Szegedy]{goodfellow2014explaining}
Ian~J Goodfellow, Jonathon Shlens, and Christian Szegedy.
\newblock Explaining and harnessing adversarial examples.
\newblock \emph{arXiv preprint arXiv:1412.6572}, 2014.

\bibitem[Gould et~al.(2016)Gould, Fernando, Cherian, Anderson, Cruz, and
  Guo]{gould2016differentiating}
Stephen Gould, Basura Fernando, Anoop Cherian, Peter Anderson, Rodrigo~Santa
  Cruz, and Edison Guo.
\newblock On differentiating parameterized argmin and argmax problems with
  application to bi-level optimization.
\newblock \emph{arXiv preprint arXiv:1607.05447}, 2016.

\bibitem[Gowal et~al.(2021)Gowal, Rebuffi, Wiles, Stimberg, Calian, and
  Mann]{gowal2021improving}
Sven Gowal, Sylvestre-Alvise Rebuffi, Olivia Wiles, Florian Stimberg,
  Dan~Andrei Calian, and Timothy~A Mann.
\newblock Improving robustness using generated data.
\newblock \emph{Advances in Neural Information Processing Systems},
  34:\penalty0 4218--4233, 2021.

\bibitem[Grazzi et~al.(2020{\natexlab{a}})Grazzi, Franceschi, Pontil, and
  Salzo]{grazzi2020bo}
Riccardo Grazzi, Luca Franceschi, Massimiliano Pontil, and Saverio Salzo.
\newblock On the iteration complexity of hypergradient computation.
\newblock \emph{International Conference on Machine Learning (ICML))},
  2020{\natexlab{a}}.

\bibitem[Grazzi et~al.(2020{\natexlab{b}})Grazzi, Franceschi, Pontil, and
  Salzo]{grazzi2020iteration}
Riccardo Grazzi, Luca Franceschi, Massimiliano Pontil, and Saverio Salzo.
\newblock On the iteration complexity of hypergradient computation.
\newblock In \emph{Proc. International Conference on Machine Learning (ICML)},
  2020{\natexlab{b}}.

\bibitem[Hu et~al.(2019)Hu, Li, Lian, Liu, and Yuan]{hu2019efficient}
Wenqing Hu, Chris~Junchi Li, Xiangru Lian, Ji~Liu, and Huizhuo Yuan.
\newblock Efficient smooth non-convex stochastic compositional optimization via
  stochastic recursive gradient descent.
\newblock \emph{Advances in Neural Information Processing Systems}, 32, 2019.

\bibitem[Ji \& Liang(2021)Ji and Liang]{ji2021lower}
Kaiyi Ji and Yingbin Liang.
\newblock Lower bounds and accelerated algorithms for bilevel optimization.
\newblock \emph{arXiv preprint arXiv:2102.03926}, 2021.

\bibitem[Ji et~al.(2020)Ji, Lee, Liang, and Poor]{ji2020convergence}
Kaiyi Ji, Jason~D Lee, Yingbin Liang, and H~Vincent Poor.
\newblock Convergence of meta-learning with task-specific adaptation over
  partial parameter.
\newblock In \emph{Advances in Neural Information Processing Systems
  (NeurIPS)}, 2020.

\bibitem[Ji et~al.(2021)Ji, Yang, and Liang]{ji2021bo}
Kayi Ji, Junjie Yang, and Yingbin Liang.
\newblock Bilevel optimization: Convergence analysis and enhanced design.
\newblock \emph{International Conference on Machine Learning (ICML))}, 2021.

\bibitem[Jiang \& Zhai(2007)Jiang and Zhai]{jiang2007instance}
Jing Jiang and ChengXiang Zhai.
\newblock Instance weighting for domain adaptation in nlp.
\newblock ACL, 2007.

\bibitem[Krizhevsky \& Hinton(2009)Krizhevsky and
  Hinton]{krizhevsky2009learning}
Alex Krizhevsky and Geoffrey Hinton.
\newblock Learning multiple layers of features from tiny images.
\newblock 2009.

\bibitem[Lian et~al.(2017)Lian, Wang, and Liu]{lian2017finite}
Xiangru Lian, Mengdi Wang, and Ji~Liu.
\newblock Finite-sum composition optimization via variance reduced gradient
  descent.
\newblock In \emph{Artificial Intelligence and Statistics}, pp.\  1159--1167.
  PMLR, 2017.

\bibitem[Liao et~al.(2018)Liao, Xiong, Fetaya, Zhang, Yoon, Urtasun, and
  Zemel]{liao2018opt}
Renjie Liao, Yuwen Xiong, Ethan Fetaya, Lisa Zhang, Xaq Yoon, KiJung~Pitkow,
  Raquel Urtasun, and Richard Zemel.
\newblock Reviving and improving recurrent back-propagation.
\newblock \emph{International Conference on Machine Learning (ICML))}, 2018.

\bibitem[Lin et~al.(2020)Lin, Fan, Wang, and Jordan]{lin2020improved}
Tianyi Lin, Chengyou Fan, Mengdi Wang, and Michael~I Jordan.
\newblock Improved sample complexity for stochastic compositional variance
  reduced gradient.
\newblock In \emph{2020 American Control Conference (ACC)}, pp.\  126--131.
  IEEE, 2020.

\bibitem[Liu et~al.(2021{\natexlab{a}})Liu, Han, Liu, Gong, Niu, Zhou,
  Sugiyama, et~al.]{liu2021probabilistic}
Feng Liu, Bo~Han, Tongliang Liu, Chen Gong, Gang Niu, Mingyuan Zhou, Masashi
  Sugiyama, et~al.
\newblock Probabilistic margins for instance reweighting in adversarial
  training.
\newblock \emph{Advances in Neural Information Processing Systems},
  34:\penalty0 23258--23269, 2021{\natexlab{a}}.

\bibitem[Liu et~al.(2018)Liu, Simonyan, and Yang]{liu2018darts}
Hanxiao Liu, Karen Simonyan, and Yiming Yang.
\newblock Darts: Differentiable architecture search.
\newblock \emph{arXiv preprint arXiv:1806.09055}, 2018.

\bibitem[Liu et~al.(2020)Liu, Mu, Yuan, Zeng, and Zhang]{liu2020generic}
Risheng Liu, Pan Mu, Xiaoming Yuan, Shangzhi Zeng, and Jin Zhang.
\newblock A generic first-order algorithmic framework for bi-level programming
  beyond lower-level singleton.
\newblock In \emph{International Conference on Machine Learning (ICML)}, 2020.

\bibitem[Liu et~al.(2021{\natexlab{b}})Liu, Gao, Zhang, Meng, and
  Lin]{liu2021investigating}
Risheng Liu, Jiaxin Gao, Jin Zhang, Deyu Meng, and Zhouchen Lin.
\newblock Investigating bi-level optimization for learning and vision from a
  unified perspective: A survey and beyond.
\newblock \emph{arXiv preprint arXiv:2101.11517}, 2021{\natexlab{b}}.

\bibitem[Liu et~al.(2021{\natexlab{c}})Liu, Liu, Zeng, and
  Zhang]{liu2021towards}
Risheng Liu, Yaohua Liu, Shangzhi Zeng, and Jin Zhang.
\newblock Towards gradient-based bilevel optimization with non-convex followers
  and beyond.
\newblock \emph{Advances in Neural Information Processing Systems (NeurIPS)},
  34, 2021{\natexlab{c}}.

\bibitem[Lorraine et~al.(2020)Lorraine, Vicol, and Duvenaud]{lorraine2020ho}
Jonathan Lorraine, Paul Vicol, and David Duvenaud.
\newblock Optimizing millions of hyperparameters by implicit differentiation.
\newblock \emph{International Conference on Artificial Intelligence and
  Statistics (AISTATS)}, pp.\  1540–1552, 2020.

\bibitem[Maclaurin et~al.(2015)Maclaurin, Duvenaud, and
  Adams]{maclaurin2015gradient}
Dougal Maclaurin, David Duvenaud, and Ryan Adams.
\newblock Gradient-based hyperparameter optimization through reversible
  learning.
\newblock In \emph{International Conference on Machine Learning (ICML)}, pp.\
  2113--2122, 2015.

\bibitem[Madry et~al.(2017)Madry, Makelov, Schmidt, Tsipras, and
  Vladu]{madry2017towards}
Aleksander Madry, Aleksandar Makelov, Ludwig Schmidt, Dimitris Tsipras, and
  Adrian Vladu.
\newblock Towards deep learning models resistant to adversarial attacks.
\newblock \emph{arXiv preprint arXiv:1706.06083}, 2017.

\bibitem[Moosavi-Dezfooli et~al.(2019)Moosavi-Dezfooli, Fawzi, Uesato, and
  Frossard]{moosavi2019robustness}
Seyed-Mohsen Moosavi-Dezfooli, Alhussein Fawzi, Jonathan Uesato, and Pascal
  Frossard.
\newblock Robustness via curvature regularization, and vice versa.
\newblock In \emph{Proceedings of the IEEE/CVF Conference on Computer Vision
  and Pattern Recognition}, pp.\  9078--9086, 2019.

\bibitem[Netzer et~al.(2011)Netzer, Wang, Coates, Bissacco, Wu, and
  Ng]{netzer2011reading}
Yuval Netzer, Tao Wang, Adam Coates, Alessandro Bissacco, Bo~Wu, and Andrew~Y
  Ng.
\newblock Reading digits in natural images with unsupervised feature learning.
\newblock 2011.

\bibitem[Nguyen et~al.(2015)Nguyen, Yosinski, and Clune]{nguyen2015deep}
Anh Nguyen, Jason Yosinski, and Jeff Clune.
\newblock Deep neural networks are easily fooled: High confidence predictions
  for unrecognizable images.
\newblock In \emph{Proceedings of the IEEE conference on computer vision and
  pattern recognition}, pp.\  427--436, 2015.

\bibitem[Papernot et~al.(2016)Papernot, McDaniel, Jha, Fredrikson, Celik, and
  Swami]{papernot2016limitations}
Nicolas Papernot, Patrick McDaniel, Somesh Jha, Matt Fredrikson, Z~Berkay
  Celik, and Ananthram Swami.
\newblock The limitations of deep learning in adversarial settings.
\newblock In \emph{2016 IEEE European symposium on security and privacy
  (EuroS\&P)}, pp.\  372--387. IEEE, 2016.

\bibitem[Pedregosa(2016)]{pedregosa2016hyperparameter}
Fabian Pedregosa.
\newblock Hyperparameter optimization with approximate gradient.
\newblock In \emph{International Conference on Machine Learning (ICML)}, pp.\
  737--746, 2016.

\bibitem[Pethick et~al.(2023)Pethick, Chrysos, and
  Cevher]{pethick2023revisiting}
Thomas Pethick, Grigorios~G Chrysos, and Volkan Cevher.
\newblock Revisiting adversarial training for the worst-performing class.
\newblock \emph{arXiv preprint arXiv:2302.08872}, 2023.

\bibitem[Qi et~al.(2021)Qi, Guo, Xu, Jin, and Yang]{qi2021online}
Qi~Qi, Zhishuai Guo, Yi~Xu, Rong Jin, and Tianbao Yang.
\newblock An online method for a class of distributionally robust optimization
  with non-convex objectives.
\newblock \emph{Advances in Neural Information Processing Systems},
  34:\penalty0 10067--10080, 2021.

\bibitem[Qian et~al.(2019)Qian, Zhu, Tang, Jin, Sun, and Li]{qian2019robust}
Qi~Qian, Shenghuo Zhu, Jiasheng Tang, Rong Jin, Baigui Sun, and Hao Li.
\newblock Robust optimization over multiple domains.
\newblock In \emph{Proceedings of the AAAI Conference on Artificial
  Intelligence}, volume~33, pp.\  4739--4746, 2019.

\bibitem[Rahimian \& Mehrotra(2019)Rahimian and
  Mehrotra]{rahimian2019distributionally}
Hamed Rahimian and Sanjay Mehrotra.
\newblock Distributionally robust optimization: A review.
\newblock \emph{arXiv preprint arXiv:1908.05659}, 2019.

\bibitem[Rajeswaran et~al.(2019)Rajeswaran, Finn, Kakade, and
  Levine]{rajeswaran2019meta}
Aravind Rajeswaran, Chelsea Finn, Sham~M Kakade, and Sergey Levine.
\newblock Meta-learning with implicit gradients.
\newblock In \emph{Advances in Neural Information Processing Systems
  (NeurIPS)}, pp.\  113--124, 2019.

\bibitem[Rebuffi et~al.(2021)Rebuffi, Gowal, Calian, Stimberg, Wiles, and
  Mann]{rebuffi2021data}
Sylvestre-Alvise Rebuffi, Sven Gowal, Dan~Andrei Calian, Florian Stimberg,
  Olivia Wiles, and Timothy~A Mann.
\newblock Data augmentation can improve robustness.
\newblock \emph{Advances in Neural Information Processing Systems},
  34:\penalty0 29935--29948, 2021.

\bibitem[Ren et~al.(2018)Ren, Zeng, Yang, and Urtasun]{ren2018learning}
Mengye Ren, Wenyuan Zeng, Bin Yang, and Raquel Urtasun.
\newblock Learning to reweight examples for robust deep learning.
\newblock In \emph{International conference on machine learning}, pp.\
  4334--4343. PMLR, 2018.

\bibitem[Rice et~al.(2020)Rice, Wong, and Kolter]{rice2020overfitting}
Leslie Rice, Eric Wong, and Zico Kolter.
\newblock Overfitting in adversarially robust deep learning.
\newblock In \emph{International Conference on Machine Learning}, pp.\
  8093--8104. PMLR, 2020.

\bibitem[Shaban et~al.(2019)Shaban, Cheng, Hatch, and
  Boots]{shaban2019truncated}
Amirreza Shaban, Ching-An Cheng, Nathan Hatch, and Byron Boots.
\newblock Truncated back-propagation for bilevel optimization.
\newblock In \emph{International Conference on Artificial Intelligence and
  Statistics (AISTATS)}, pp.\  1723--1732, 2019.

\bibitem[Stallkamp et~al.(2012)Stallkamp, Schlipsing, Salmen, and
  Igel]{stallkamp2012man}
Johannes Stallkamp, Marc Schlipsing, Jan Salmen, and Christian Igel.
\newblock Man vs. computer: Benchmarking machine learning algorithms for
  traffic sign recognition.
\newblock \emph{Neural networks}, 32:\penalty0 323--332, 2012.

\bibitem[Szegedy et~al.(2013)Szegedy, Zaremba, Sutskever, Bruna, Erhan,
  Goodfellow, and Fergus]{szegedy2013intriguing}
Christian Szegedy, Wojciech Zaremba, Ilya Sutskever, Joan Bruna, Dumitru Erhan,
  Ian Goodfellow, and Rob Fergus.
\newblock Intriguing properties of neural networks.
\newblock \emph{arXiv preprint arXiv:1312.6199}, 2013.

\bibitem[Tian et~al.(2021)Tian, Kuang, Jiang, Wu, and Wang]{tian2021analysis}
Qi~Tian, Kun Kuang, Kelu Jiang, Fei Wu, and Yisen Wang.
\newblock Analysis and applications of class-wise robustness in adversarial
  training.
\newblock In \emph{Proceedings of the 27th ACM SIGKDD Conference on Knowledge
  Discovery \& Data Mining}, pp.\  1561--1570, 2021.

\bibitem[Tram{\`e}r et~al.(2017)Tram{\`e}r, Papernot, Goodfellow, Boneh, and
  McDaniel]{tramer2017space}
Florian Tram{\`e}r, Nicolas Papernot, Ian Goodfellow, Dan Boneh, and Patrick
  McDaniel.
\newblock The space of transferable adversarial examples.
\newblock \emph{arXiv preprint arXiv:1704.03453}, 2017.

\bibitem[Tutunov et~al.(2020)Tutunov, Li, Cowen-Rivers, Wang, and
  Bou-Ammar]{tutunov2020compositional}
Rasul Tutunov, Minne Li, Alexander~I Cowen-Rivers, Jun Wang, and Haitham
  Bou-Ammar.
\newblock Compositional adam: An adaptive compositional solver.
\newblock \emph{arXiv preprint arXiv:2002.03755}, 2020.

\bibitem[Wang et~al.(2021)Wang, Zhang, Liu, Chen, Xu, Fardad, and
  Li]{wang2021adversarial}
Jingkang Wang, Tianyun Zhang, Sijia Liu, Pin-Yu Chen, Jiacen Xu, Makan Fardad,
  and Bo~Li.
\newblock Adversarial attack generation empowered by min-max optimization.
\newblock \emph{Advances in Neural Information Processing Systems},
  34:\penalty0 16020--16033, 2021.

\bibitem[Wang et~al.(2016)Wang, Liu, and Fang]{wang2016accelerating}
Mengdi Wang, Ji~Liu, and Ethan Fang.
\newblock Accelerating stochastic composition optimization.
\newblock \emph{Advances in Neural Information Processing Systems}, 29, 2016.

\bibitem[Wang et~al.(2017)Wang, Fang, and Liu]{wang2017stochastic}
Mengdi Wang, Ethan~X Fang, and Han Liu.
\newblock Stochastic compositional gradient descent: algorithms for minimizing
  compositions of expected-value functions.
\newblock \emph{Mathematical Programming}, 161:\penalty0 419--449, 2017.

\bibitem[Wang et~al.(2022)Wang, Xu, Liu, Li, Thuraisingham, and
  Tang]{wang2022imbalanced}
Wentao Wang, Han Xu, Xiaorui Liu, Yaxin Li, Bhavani Thuraisingham, and Jiliang
  Tang.
\newblock Imbalanced adversarial training with reweighting.
\newblock In \emph{2022 IEEE International Conference on Data Mining (ICDM)},
  pp.\  1209--1214. IEEE, 2022.

\bibitem[Wong et~al.(2020)Wong, Rice, and Kolter]{wong2020fast}
Eric Wong, Leslie Rice, and J~Zico Kolter.
\newblock Fast is better than free: Revisiting adversarial training.
\newblock \emph{arXiv preprint arXiv:2001.03994}, 2020.

\bibitem[Wu et~al.(2021)Wu, Liu, Huang, Wang, and Lin]{wu2021adversarial}
Tong Wu, Ziwei Liu, Qingqiu Huang, Yu~Wang, and Dahua Lin.
\newblock Adversarial robustness under long-tailed distribution.
\newblock In \emph{Proceedings of the IEEE/CVF conference on computer vision
  and pattern recognition}, pp.\  8659--8668, 2021.

\bibitem[Yang et~al.(2020{\natexlab{a}})Yang, Rashtchian, Zhang, Salakhutdinov,
  and Chaudhuri]{yang2020adversarial}
Yao-Yuan Yang, Cyrus Rashtchian, Hongyang Zhang, Ruslan Salakhutdinov, and
  Kamalika Chaudhuri.
\newblock Adversarial robustness through local lipschitzness.
\newblock \emph{arXiv preprint arXiv:2003.02460}, 20, 2020{\natexlab{a}}.

\bibitem[Yang et~al.(2020{\natexlab{b}})Yang, Rashtchian, Zhang, Salakhutdinov,
  and Chaudhuri]{yang2020closer}
Yao-Yuan Yang, Cyrus Rashtchian, Hongyang Zhang, Russ~R Salakhutdinov, and
  Kamalika Chaudhuri.
\newblock A closer look at accuracy vs. robustness.
\newblock \emph{Advances in neural information processing systems},
  33:\penalty0 8588--8601, 2020{\natexlab{b}}.

\bibitem[Yi et~al.(2021)Yi, Hou, Shang, Jiang, Liu, and Ma]{yi2021reweighting}
Mingyang Yi, Lu~Hou, Lifeng Shang, Xin Jiang, Qun Liu, and Zhi-Ming Ma.
\newblock Reweighting augmented samples by minimizing the maximal expected
  loss.
\newblock \emph{arXiv preprint arXiv:2103.08933}, 2021.

\bibitem[Zeng et~al.(2021)Zeng, Zhu, Goldstein, and Huang]{zeng2021adversarial}
Huimin Zeng, Chen Zhu, Tom Goldstein, and Furong Huang.
\newblock Are adversarial examples created equal? a learnable weighted minimax
  risk for robustness under non-uniform attacks.
\newblock In \emph{Proceedings of the AAAI Conference on Artificial
  Intelligence}, volume~35, pp.\  10815--10823, 2021.

\bibitem[Zhang et~al.(2019)Zhang, Yu, Jiao, Xing, El~Ghaoui, and
  Jordan]{zhang2019theoretically}
Hongyang Zhang, Yaodong Yu, Jiantao Jiao, Eric Xing, Laurent El~Ghaoui, and
  Michael Jordan.
\newblock Theoretically principled trade-off between robustness and accuracy.
\newblock In \emph{International conference on machine learning}, pp.\
  7472--7482. PMLR, 2019.

\bibitem[Zhang et~al.(2020{\natexlab{a}})Zhang, Xu, Han, Niu, Cui, Sugiyama,
  and Kankanhalli]{zhang2020attacks}
Jingfeng Zhang, Xilie Xu, Bo~Han, Gang Niu, Lizhen Cui, Masashi Sugiyama, and
  Mohan Kankanhalli.
\newblock Attacks which do not kill training make adversarial learning
  stronger.
\newblock In \emph{International conference on machine learning}, pp.\
  11278--11287. PMLR, 2020{\natexlab{a}}.

\bibitem[Zhang et~al.(2020{\natexlab{b}})Zhang, Zhu, Niu, Han, Sugiyama, and
  Kankanhalli]{zhang2020geometry}
Jingfeng Zhang, Jianing Zhu, Gang Niu, Bo~Han, Masashi Sugiyama, and Mohan
  Kankanhalli.
\newblock Geometry-aware instance-reweighted adversarial training.
\newblock \emph{arXiv preprint arXiv:2010.01736}, 2020{\natexlab{b}}.

\bibitem[Zhang et~al.(2021{\natexlab{a}})Zhang, Su, Pan, Chang, Abbasnejad, and
  Haffari]{zhang2021idarts}
Miao Zhang, Steven Su, Shirui Pan, Xiaojun Chang, Ehsan Abbasnejad, and Reza
  Haffari.
\newblock idarts: Differentiable architecture search with stochastic implicit
  gradients.
\newblock \emph{arXiv preprint arXiv:2106.10784}, 2021{\natexlab{a}}.

\bibitem[Zhang et~al.(2022)Zhang, Zhang, Khanduri, Hong, Chang, and
  Liu]{zhang2022revisiting}
Yihua Zhang, Guanhua Zhang, Prashant Khanduri, Mingyi Hong, Shiyu Chang, and
  Sijia Liu.
\newblock Revisiting and advancing fast adversarial training through the lens
  of bi-level optimization.
\newblock In \emph{International Conference on Machine Learning}, pp.\
  26693--26712. PMLR, 2022.

\bibitem[Zhang et~al.(2021{\natexlab{b}})Zhang, Gong, Liu, Niu, Tian, Han,
  Sch{\"o}lkopf, and Zhang]{zhang2021causaladv}
Yonggang Zhang, Mingming Gong, Tongliang Liu, Gang Niu, Xinmei Tian, Bo~Han,
  Bernhard Sch{\"o}lkopf, and Kun Zhang.
\newblock Causaladv: Adversarial robustness through the lens of causality.
\newblock \emph{arXiv preprint arXiv:2106.06196}, 2021{\natexlab{b}}.

\end{thebibliography}
\bibliographystyle{iclr2021_conference}

\newpage
\appendix 

\textbf{\Large Supplementary Material}
\vspace{1cm}

\noindent We provide the details omitted in the main paper. The sections are organized as fellows: 

\noindent $\bullet$ \Cref{app:exp}: We provide more details about datasets, training setups, hyperparameters search, and implementations. 

\noindent $\bullet$ \Cref{app:weights}: We provide the distributions of the robustly learned instance weights and models' confusion matrices for the considered datasets. 

\noindent $\bullet$ \Cref{app:implicit}: We provide the proof of \Cref{prop:igrad}. 

\noindent $\bullet$ \Cref{app:theory}: We present the convergence analysis of our proposed CID algorithm including the statements of assumptions and the full proof of \Cref{thr}.

\vspace{20pt}
\section{More Empirical Specifications} \label{app:exp} 
\vspace{10pt}
\subsection{More Details about Training and Hyperparameters Search}
Following the standard practice in adversarial training \cite{madry2017towards,liu2021probabilistic,zhang2020geometry}, we train our baselines using stochastic gradient descent with a minibtach size of 128 and a momentum of 0.9. We use ResNet-18 as the backbone network as in \cite{madry2017towards} and train our baselines for 60 epochs with a cyclic learning rate schedule where the maximum learning rate is set to 0.2 \cite{zhang2020geometry,liu2021towards} (please see \cref{fig:training}). We consider $\ell_\infty$-norm bounded adversarial perturbations with a maximum radius of $\epsilon=8/255$ both for training and testing. For the KL-divergence regularization parameter $r$ in our algorithms, we use a decayed schedule where we initially set it to 10 and decay it to 1 and 0.1, respectively at epochs 40 and 50 (see \cref{fig:training}). This setting allows our methods to start with an instance-weight distribution close to uniform at the beginning of training where the weights are less informative, and progressively emphasize more on learning a weight distribution that boosts worst-case adversarial robustness. All hyperparameters were fixed by holding out 10\% of the training data as a validation set and selecting the values that achieve the best performance on the validation set. For the reported results, we train on the full training dataset and report the performance on the testing set \cite{zhang2020geometry,liu2021probabilistic}. 

\vspace{10pt}
\subsection{Further Descriptions about Datasets} 
We consider image recognition problems and compare the performance of the baselines on four datasets: CIFAR10 \cite{krizhevsky2009learning}, SVHN \cite{netzer2011reading}, STL10 \cite{pmlr-v15-coates11a}, and GTSRB \cite{stallkamp2012man}. For CIFAR10, SVHN, and STL10 we use the training and test splits provided by Torchvision. For GTSRB, we use the splits provided in \cite{zhang2022revisiting}. 
STL10 has 10 categories that are similar to those in CIFAR10 but with larger colour images ($96 \times 96$ resolution) and less samples (500 per class for training and 800 per class for testing). 
The German Traffic Sign Recognition Benchmark (GTSRB) contains 43 classes of traffic signs, split into 39,209 training images and 12,630 test images. The images are $32 \times 32$ resolution colour. 
The dataset is highly class-imbalanced with some classes having over 2000 samples and others only 200 samples. 

\vspace{10pt}





\begin{figure}[hbt!]
    \centering
    \begin{tabular}{cc} 
    \includegraphics[width=0.49\textwidth,height=0.3\textwidth]{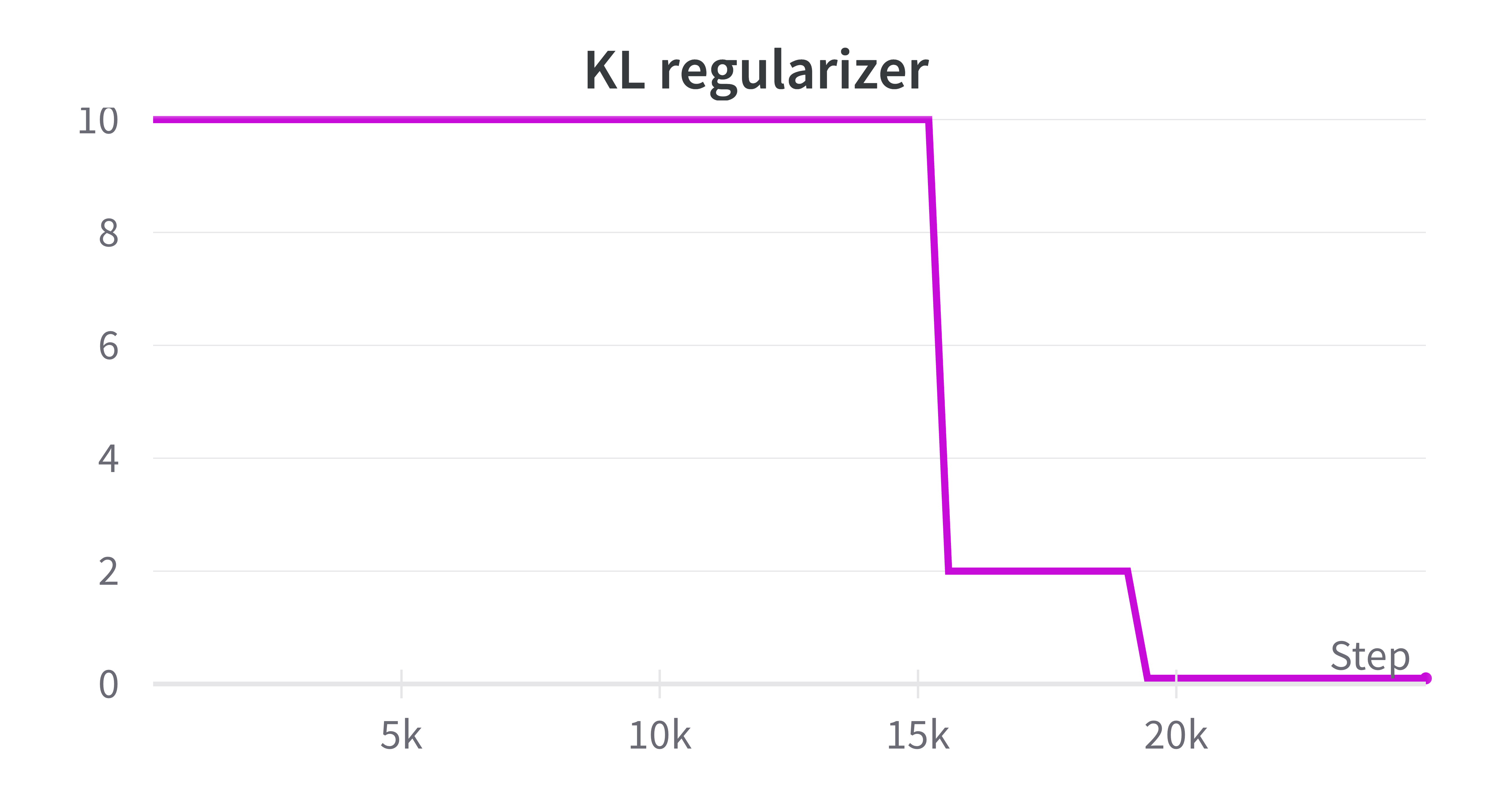} & 
    \includegraphics[width=0.49\textwidth,height=0.3\textwidth]{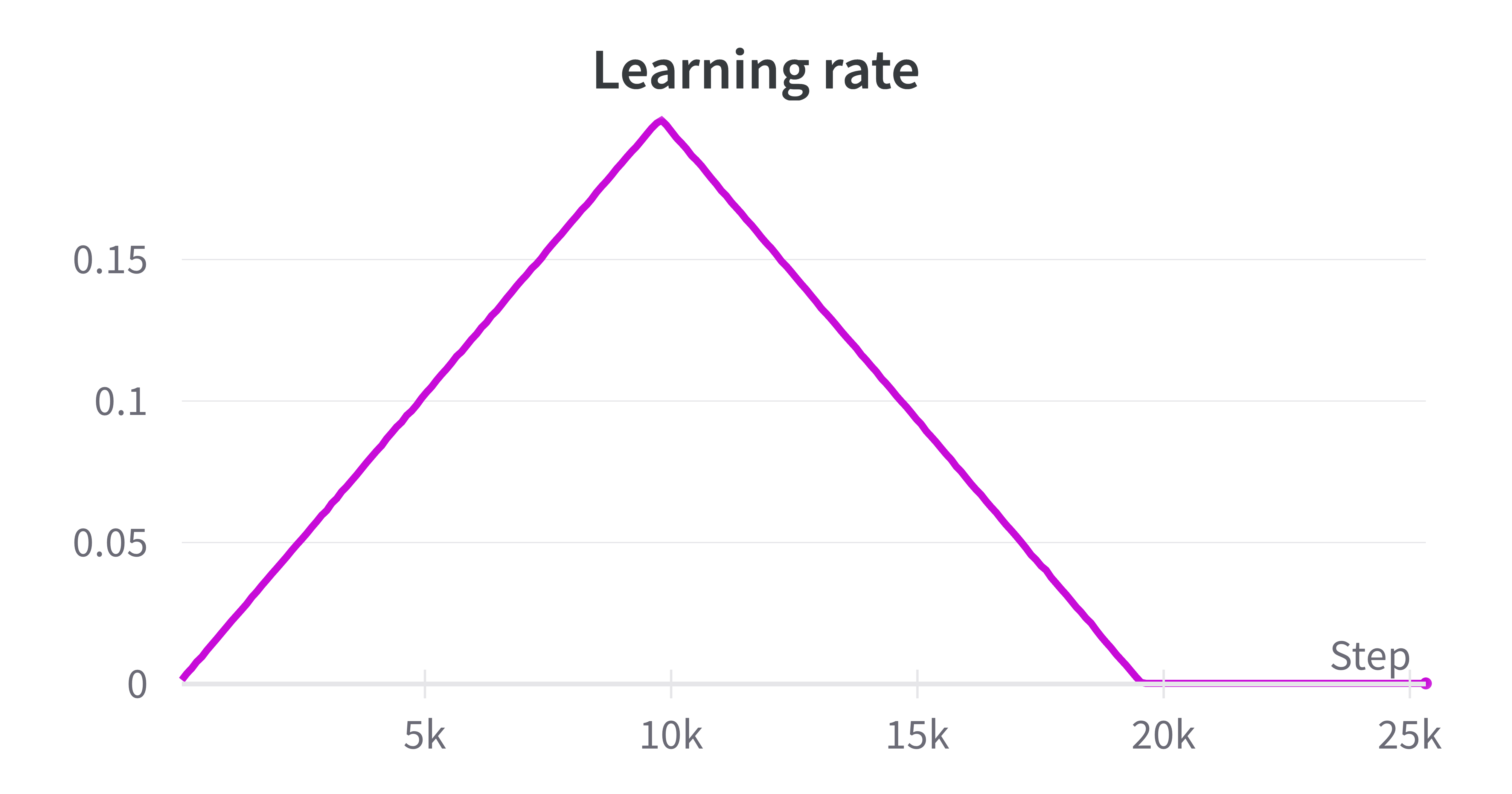} \\ 
    \includegraphics[width=0.49\textwidth,height=0.3\textwidth]{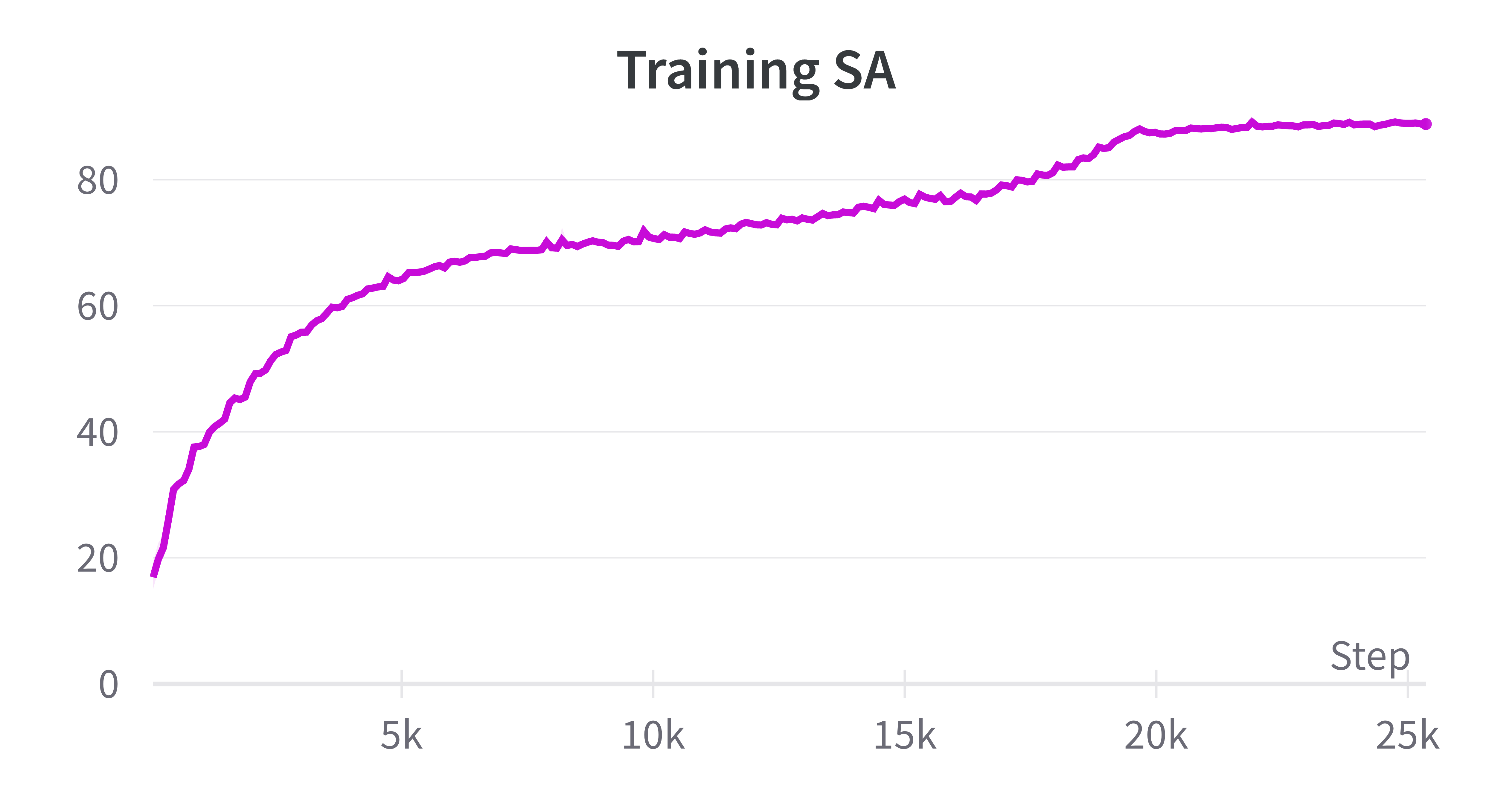} & 
    \includegraphics[width=0.49\textwidth,height=0.3\textwidth]{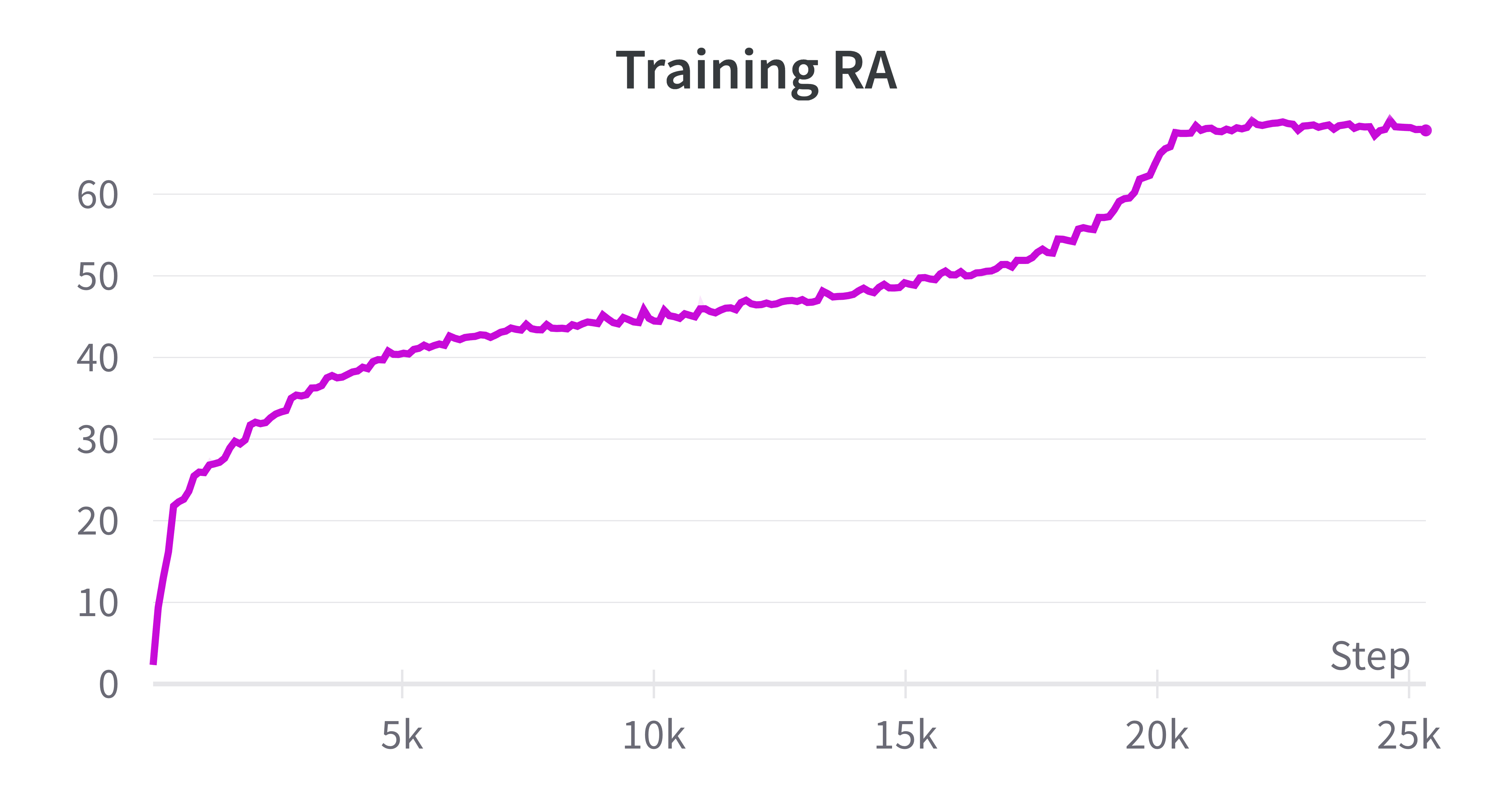} \\ 
    \includegraphics[width=0.49\textwidth,height=0.3\textwidth]{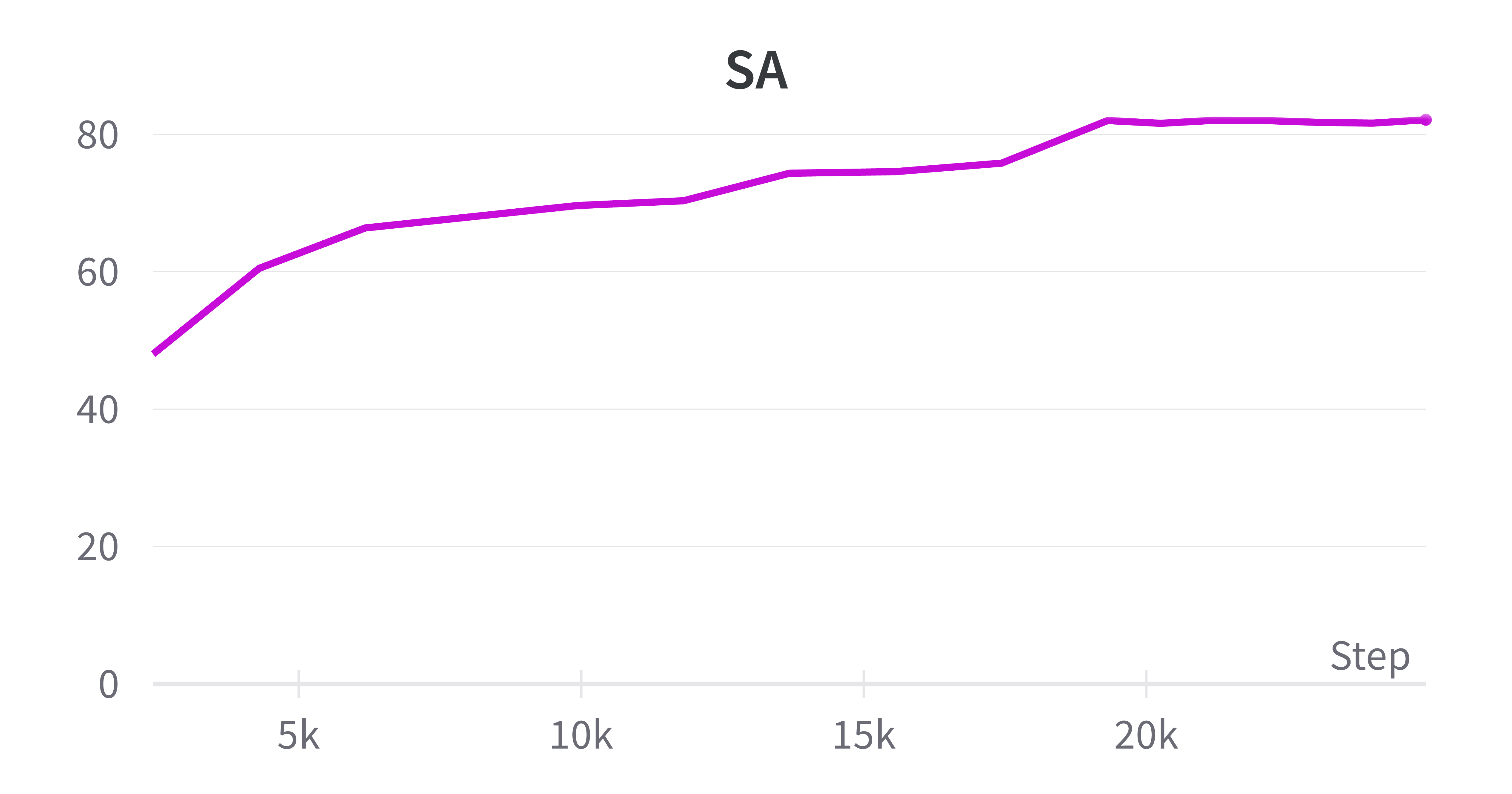} & 
    \includegraphics[width=0.49\textwidth,height=0.3\textwidth]{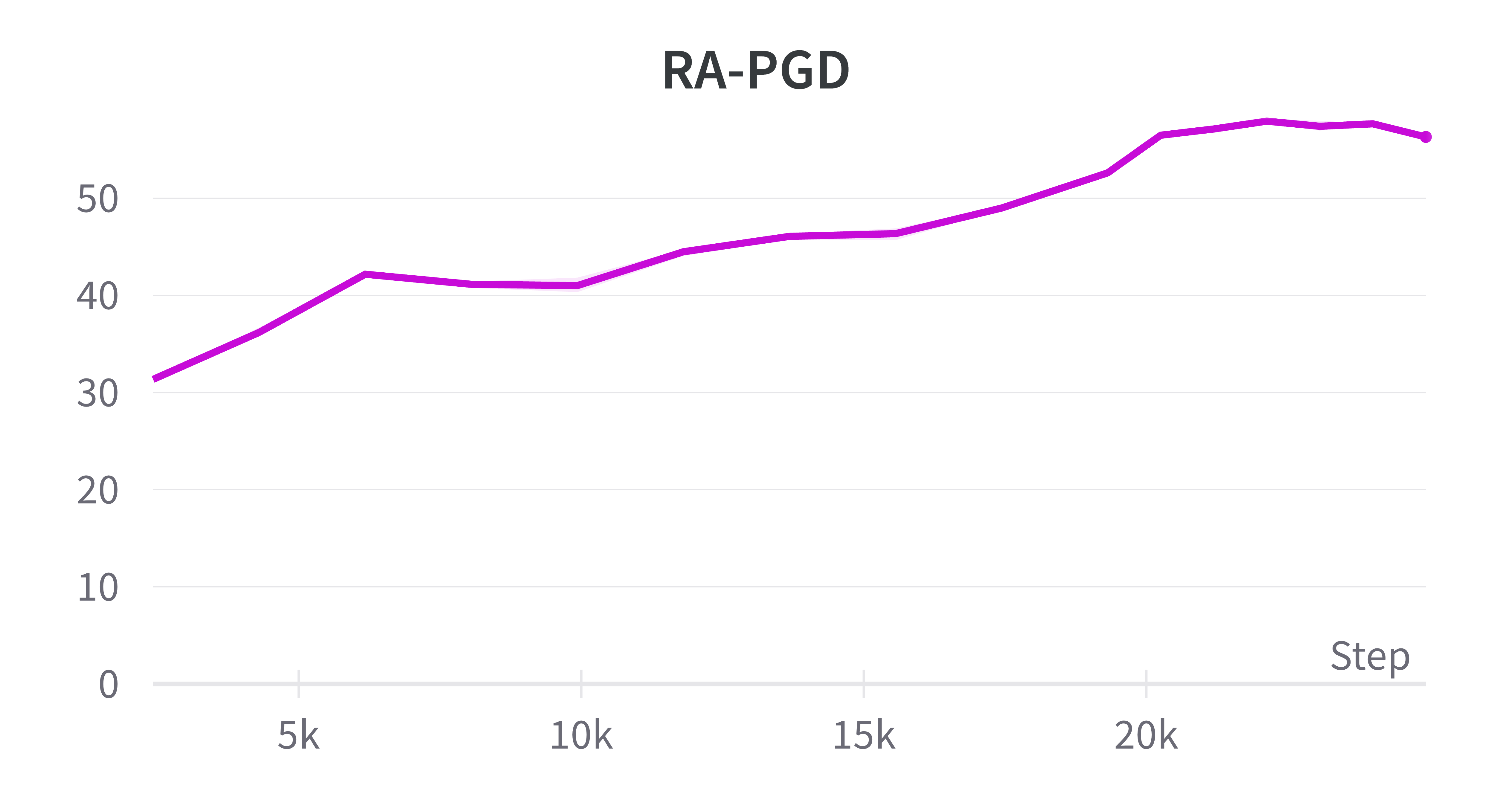} 
    \end{tabular}
    \caption{Learning process of our method DONE-ADAM for the balanced CIFAR10 experiment. The \textbf{SA} and \textbf{RA-PGD} in third row are evaluated on the test set. 
    The plots are obtained by averaging three different runs. } 
    \label{fig:training}
\end{figure}

\begin{figure}[hbt!]
    \centering
    \begin{tabular}{cc} 
    \includegraphics[width=0.5\textwidth]{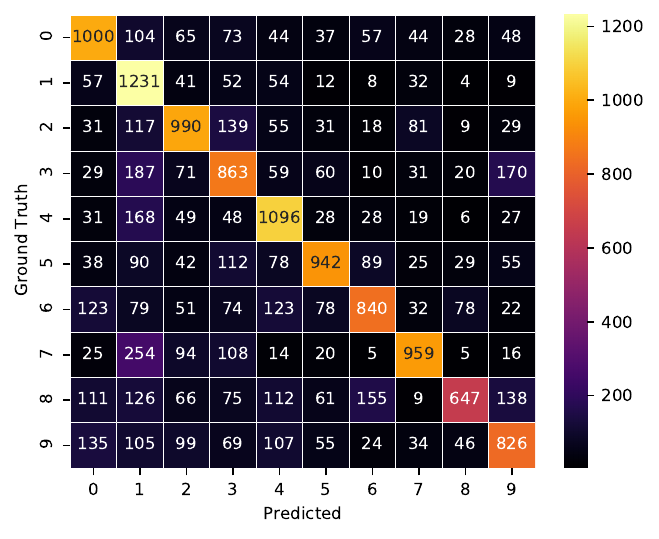} & 
    \includegraphics[width=0.5\textwidth]{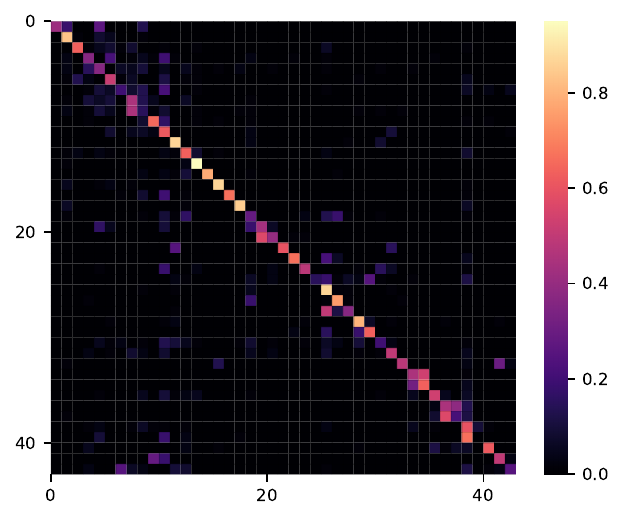} \\ 
    SVHN (divide by 1500 to obtain accuracies) & GTSRB \\ 
    \includegraphics[width=0.50\textwidth]{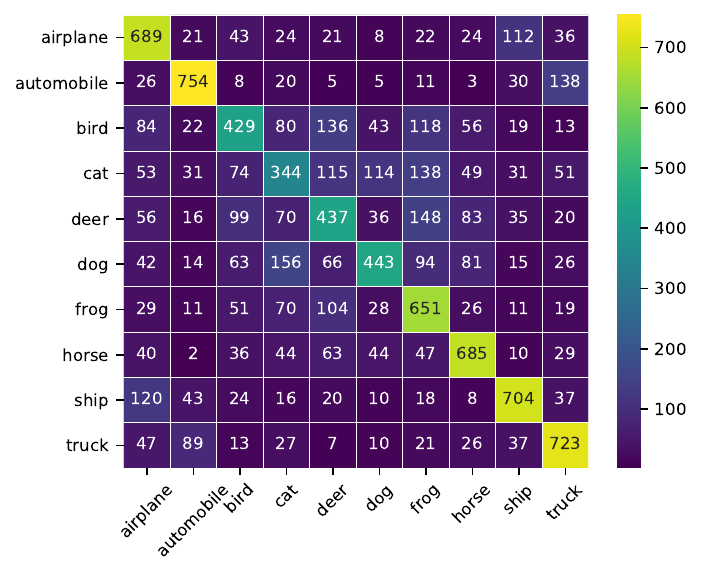} & 
    \includegraphics[width=0.5\textwidth]{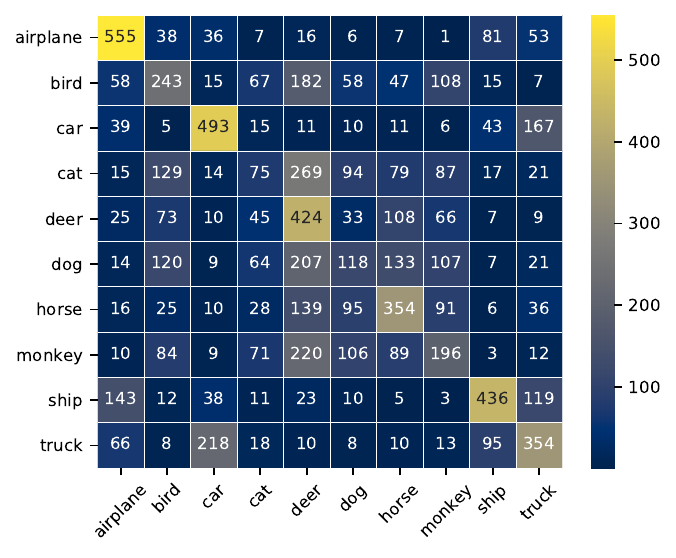} \\ 
    CIFAR10 (divide by 1000 to obtain accuracies) & STL10 (divide by 800 to obtain accuracies)
    \end{tabular}
    \caption{Confusion matrices of models robustly trained using our approach. The annotations correspond to the raw number of adversarial examples from class i that were classified as class j. Per-class robust performance are depicted in the diagonals. Axis labels are provided in first plot only.} 
    \label{fig:confmat}
\end{figure} 

\newpage
\section{Distributions of Learned Weights} \label{app:weights}
\Cref{fig:weight} shows the distributions of the learned weights per-class for CIFAR10, SVHN, and STL10 datasets. The distributions are obtained on the testing sets using 20 PGD steps. 
Further per-class insights are also provided in \Cref{fig:confmat} as the confusion matrices (where per-class robust accuracies are depicted in the diagonals). Comparing the two figures, we note a {\bf negative correlation} between the magnitude of weights and the per-class robust performance, i.e., classes on which the model achieve high robustness are usually associated with weights that are closer to 0. For example, the class \textit{automobile} in CIFAR10 datset, in which the model achieves the highest adversarial robustness of 74.5\% also has around 70\% of its associated weights less than 0.001. As a comparison, the most vulnerable class (i.e., \textit{cat}, in which the model achieves a robustness of 34.4\%) has more than 90\% of its associated weights larger than 0.001. We note a similar correlation of the weights distributions and the robust performance in STL10 dataset. Interestingly, the robust performance is more uniformly distributed across classes in the SVHN dataset (as depicted in the corresponding confusion matrix in \Cref{fig:confmat}) and our method was able to automatically discover very close weights distributions across classes for this dataset. This further demonstrates the generality/robustness of our approach, which can perform well no matter if instance re-weighting is advantageous or less important. 

\Cref{fig:easy,fig:hard} provides examples of images from CIFAR10 and STL10 datasets with low/high associated weights. Examples with low weights are usually `easy' images in which the objects are well centered with clear/non-ambiguous backgrounds. Our algorithm was able to correctly classify the adversarial examples crafted from these images. In contrast, examples with high weights are generally `hard' samples with only parts of the objects appearing or/and backgrounds that can lead to ambiguity. For example, the true label of the second image in the first row of \Cref{fig:hard} is \textit{deer} but the image also contains a car in its background, which may easily lead to confusion. Also note the first image in the third row of \Cref{fig:hard}, where only part of the tires of the car appears in the image. 

\vspace{20pt}
\begin{figure}[H]
    \centering
    \begin{tabular}{cc}
    \includegraphics[height=6.5cm]{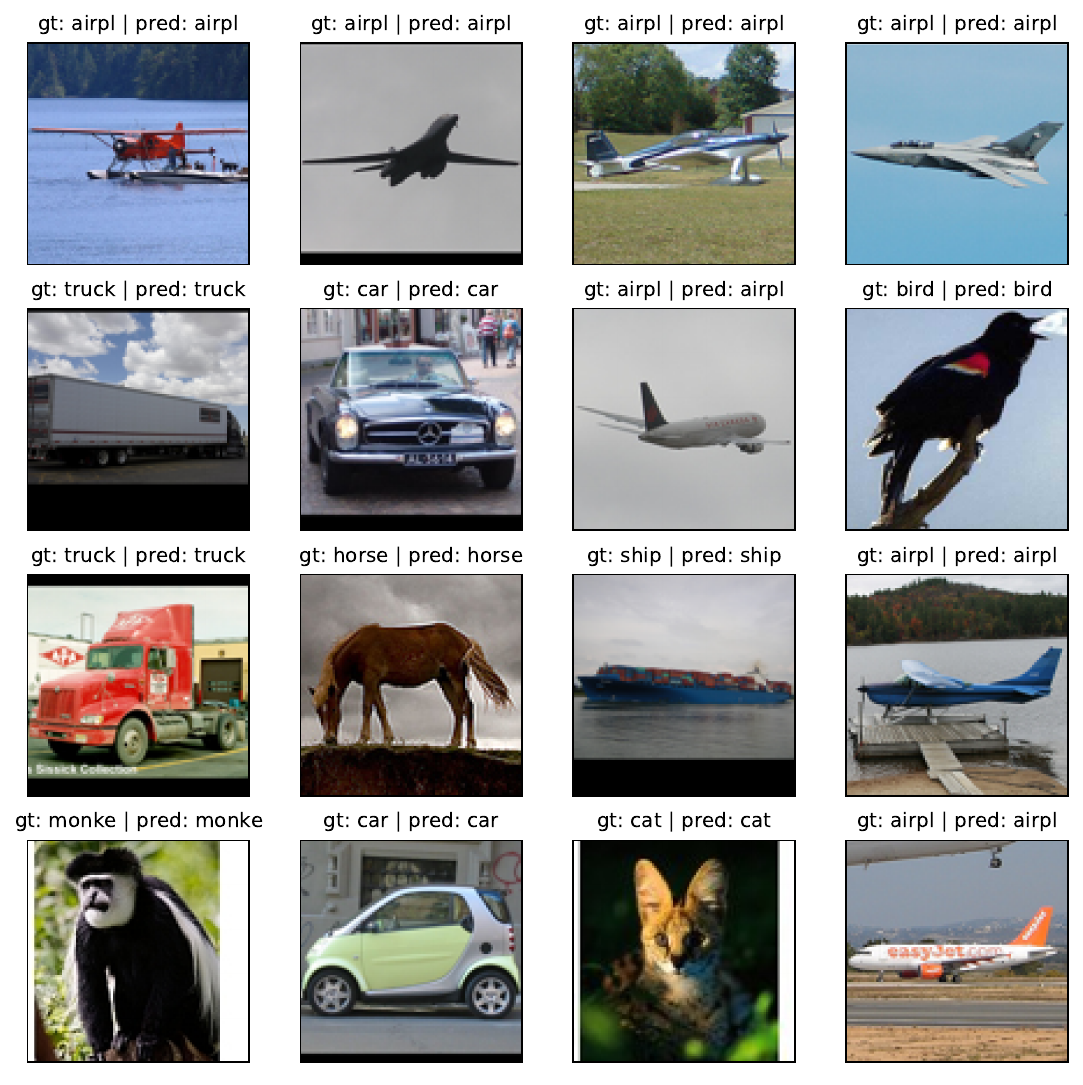}
    &\includegraphics[height=6.5cm]{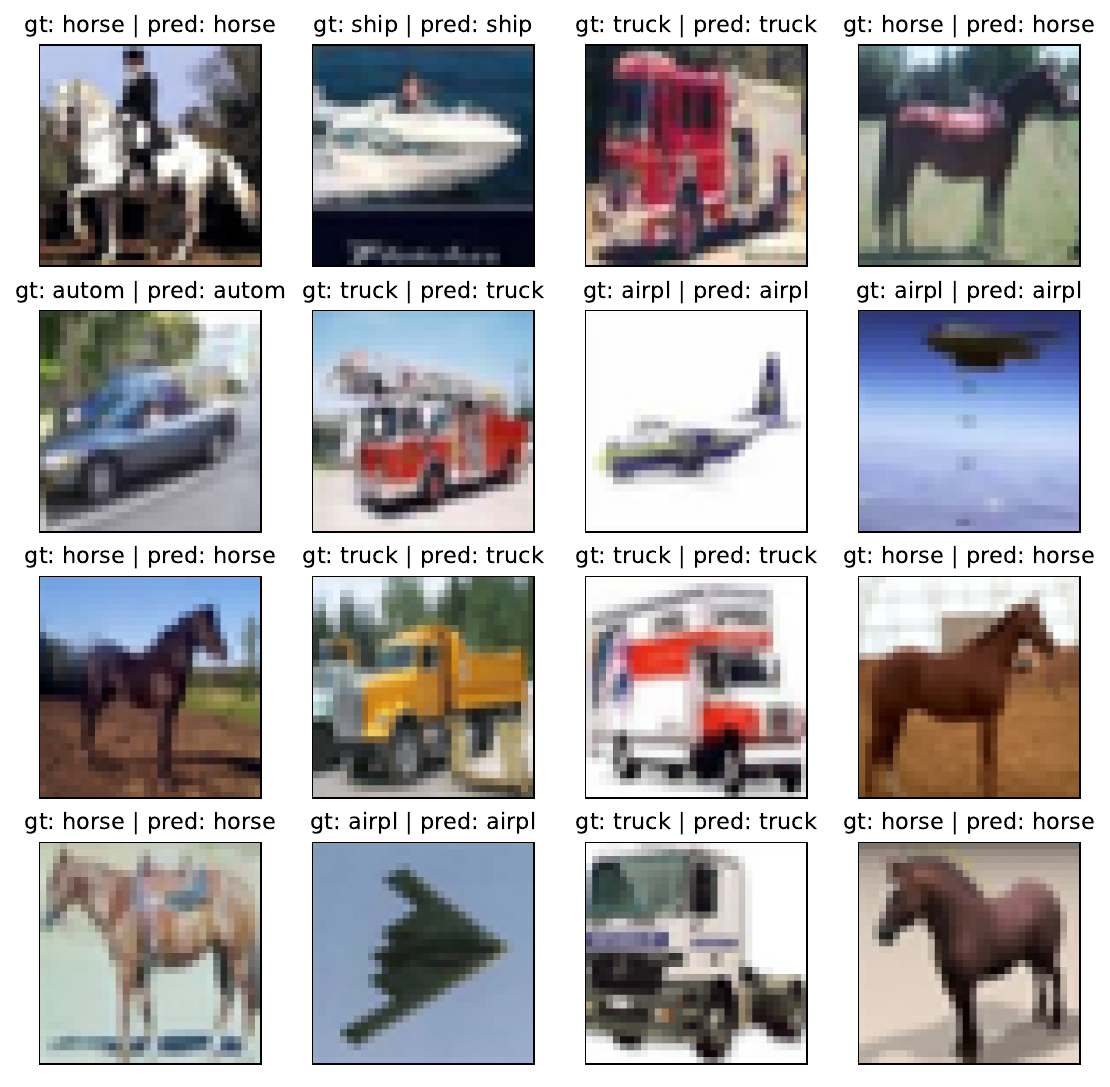} \\ 
    (a) STL10 dataset & (b) CIFAR10 dataset 
    \end{tabular}
    \caption{Samples with small weights from STL10 and CIFAR10 datasets. These are generally `easy' images with the true objects well centered and clear/non-ambiguous backgrounds.}
    \label{fig:easy}
\end{figure} 

\begin{figure}[H]
    \centering
    \begin{tabular}{c} 
    \includegraphics[width=0.85\textwidth]{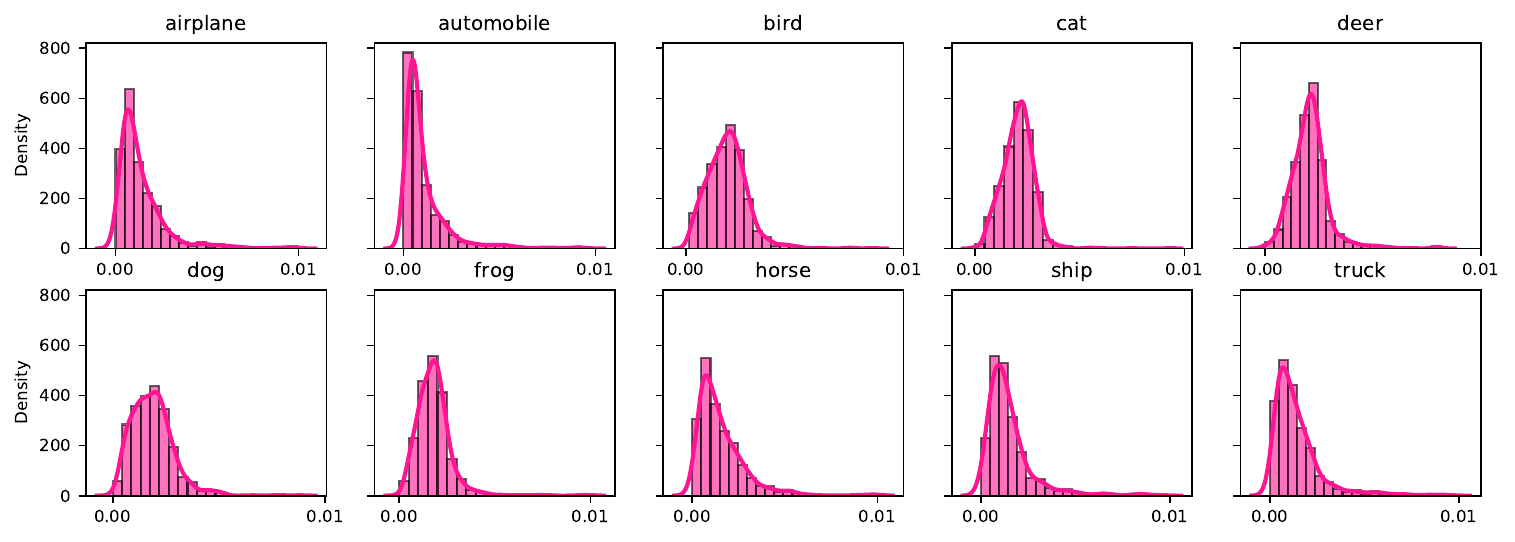}\\
    CIFAR10 dataset\\
    \includegraphics[width=0.85\textwidth]{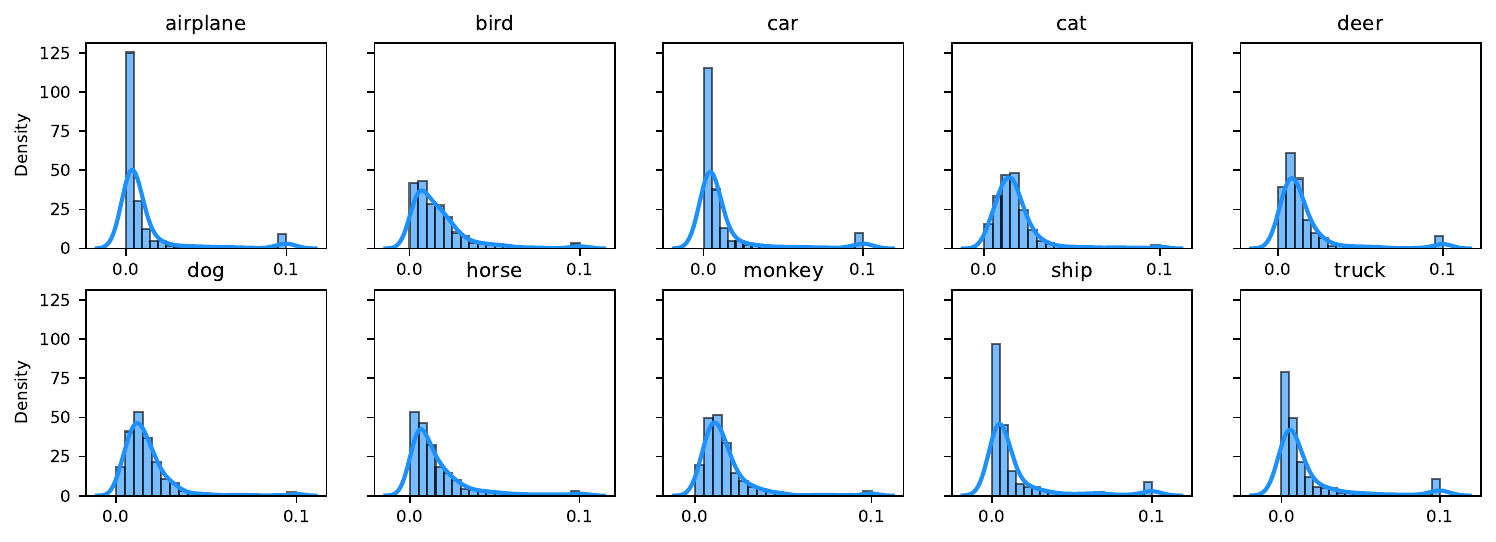}\\ 
    STL10 dataset\\
    \includegraphics[width=0.85\textwidth]{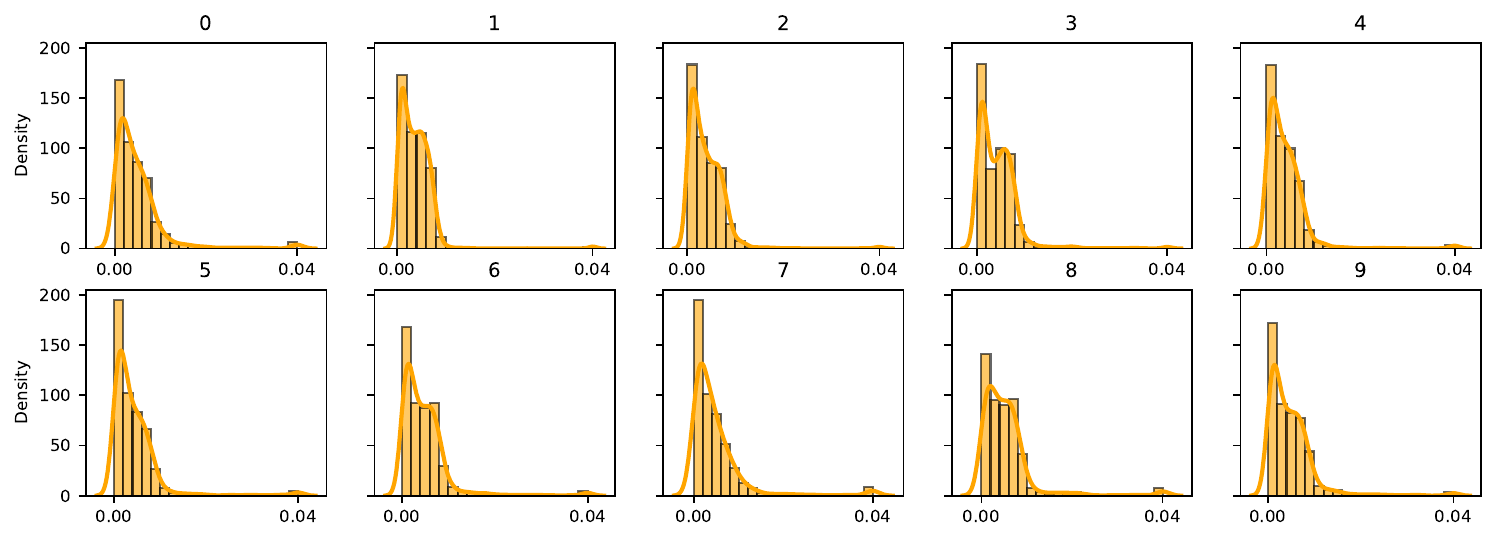} \\
    SVHN dataset\\
    \end{tabular}
    \caption{Distributions of the learned weights per class on the testing sets. Classes on which the model achieve high robustness are usually associated with weights that are closer to 0. For example, the class \textit{automobile} in CIFAR10 datset, in which the model achieves the highest adversarial robustness of 74.5\% also has around 70\% of its associated weights less than 0.001. As a comparison, the class \textit{cat} (in which the model achieves a robustness of 34.4\%) has more than 90\% of its associated weights larger than 0.001. We note a similar correlation of the weights distributions and the robust performance in STL10. The robust performance is better uniformly distributed across classes in the SVHN dataset (see \cref{fig:confmat}) and our method was able to obtain a similar weights distribution across classes for this dataset.}
    \label{fig:weight}
\end{figure}

\begin{figure}[hbt!]
    \centering
    \begin{tabular}{cc}
    \includegraphics[height=6.5cm]{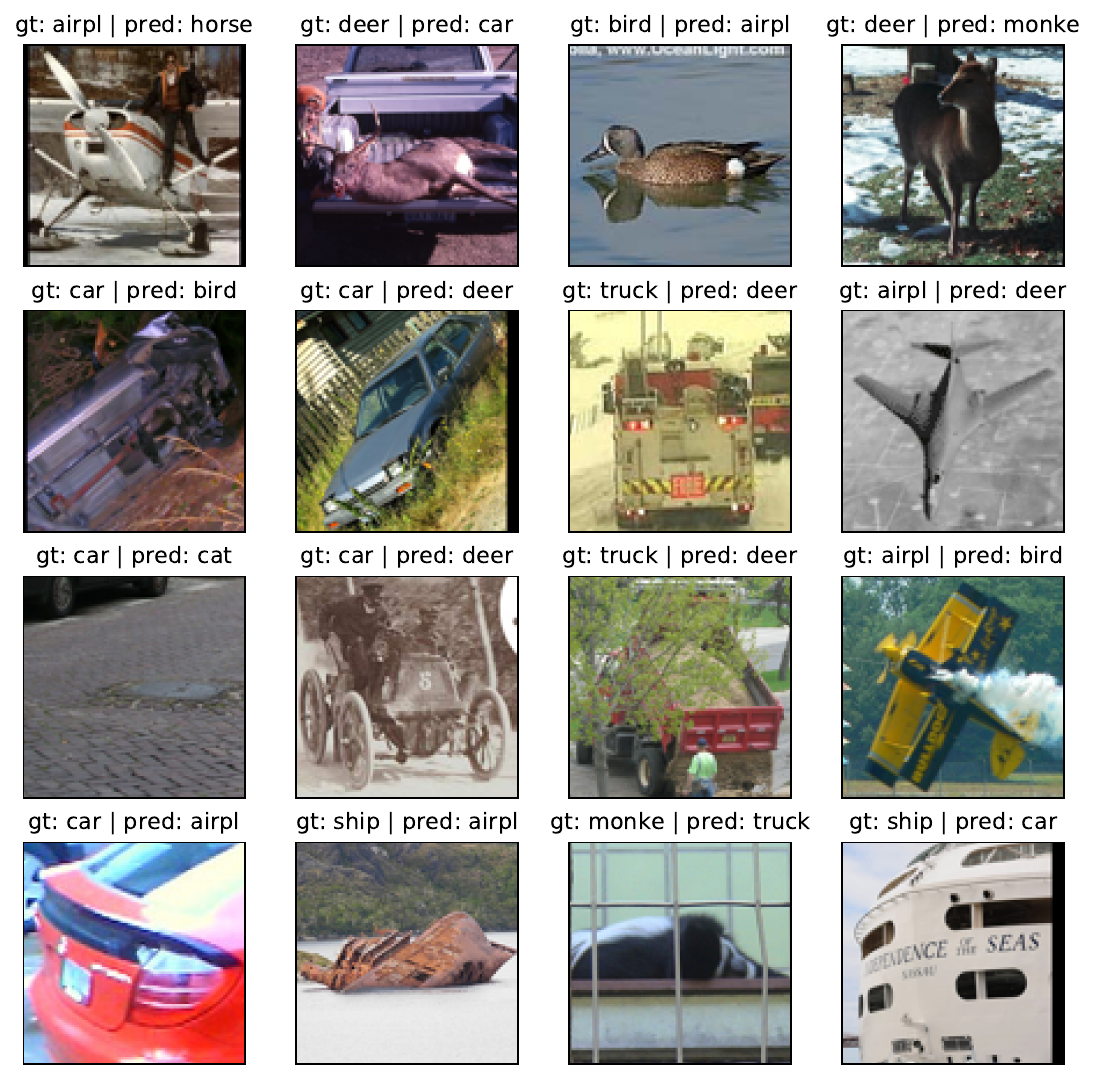}
    &\includegraphics[height=6.5cm]{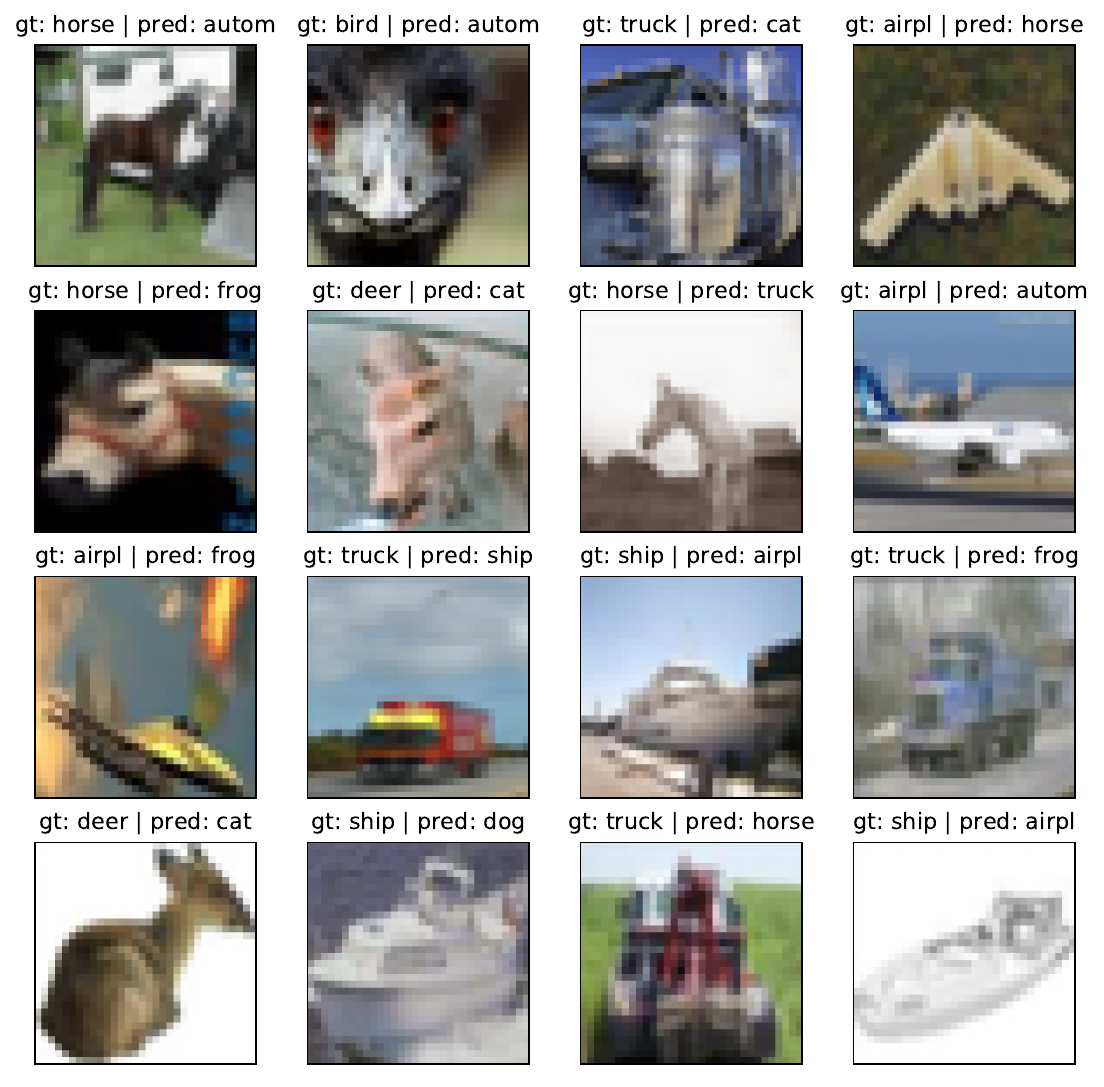} \\ 
    (a) STL10 dataset & (b) CIFAR10 dataset 
    \end{tabular}
    \caption{Samples with large weights from STL10 and CIFAR10 datasets. These are `hard' examples with only parts of the objects appearing or/and complex backgrounds that easily lead to ambiguity. For example, the true label of the second image in the first row of figure (a) is \textit{deer} but the image also contains a car in its background, which leads to ambiguity. Also note the first image in the third row of figure (a), where only part of the tires of the car appears in the image.}
    \label{fig:hard}
\end{figure} 



\section{Proof of \Cref{prop:igrad}} \label{app:implicit}

Recall the reformulated problem (\ref{eq:bar}), which we rewrite as 
\begin{align}
    &\underset{\theta}\min ~ \mathcal{L}(\theta) \coloneqq f\left(\frac{1}{M} \sum_{i=1}^M g_i (\theta, \hat \delta^*_i(\theta)) \right) \nonumber \\ 
    &\text { s.t. } \hat \delta^*_i(\theta) = \argmin_{\delta \in \mathcal{C}_i} \ell_i^{bar}(\theta, \delta) \coloneqq \ell_i'(\theta, \delta) - c \sum_{k=1}^{2p} \log (b_k - \delta^\top a_k), \nonumber 
\end{align}
where $g_i (\theta, \hat \delta^*_i(\theta)) = \exp\left(\frac{\ell_i(\theta, \hat \delta^*_i(\theta))}{r}\right)$, and $f(z) = r \log(z)$ for $z \geq 1$. 

Applying the chain rule to the outer function, we have 
\begin{align} \label{eq:chain}
    \nabla \mathcal{L}(\theta) =& \nabla f \left(\frac{1}{M} \sum_{i=1}^M g_i (\theta, \hat \delta^*_i(\theta))\right)  \frac{1}{M} \sum_{i=1}^M \frac{\partial g_i (\theta, \hat \delta^*_i(\theta))}{\partial \theta} \nonumber \\ 
    =& \frac{r}{\sum_{i=1}^M g_i (\theta, \hat \delta^*_i(\theta))} \sum_{i=1}^M \left(\nabla_\theta g_i (\theta, \hat \delta^*_i(\theta)) + \frac{\partial \hat \delta^*_i(\theta)}{\partial \theta} \nabla_\delta g_i (\theta, \hat \delta^*_i(\theta))\right).
\end{align}
Also, note that $\nabla_\delta \ell_i^{bar}(\theta, \delta) = \nabla_\delta \ell_i'(\theta, \delta) + c \sum_{k=1}^{2p} \frac{a_k}{b_k - \delta^\top a_k}$. Using the implicit differentiation w.r.t. $\theta$ of equation $\nabla_\delta \ell_i^{bar}(\theta, \hat \delta^*_i(\theta)) = \mathbf{0}$, 
i.e.,
$$
\nabla_\delta \ell_i'(\theta, \hat \delta^*_i(\theta))) + c \sum_{k=1}^{2p} \frac{a_k}{b_k - a_k^\top \hat \delta^*_i(\theta))} = \mathbf{0},$$
we obtain 
\begin{align*}
    \nabla_{\theta \delta} \ell_i'(\theta, \hat \delta^*_i(\theta)) + \frac{\partial \hat \delta^*_i(\theta)}{\partial \theta} \nabla_{\delta}^2 \ell_i'(\theta, \hat \delta^*_i(\theta)) + c \frac{\partial \hat \delta^*_i(\theta)}{\partial \theta} \sum_{k=1}^{2p} \frac{a_k a_k^\top}{\left(b_k - a_k^\top \hat \delta^*_i(\theta))\right)^2} = \mathbf{0}.
\end{align*}

Therefore, we obtain 
\begin{align} \label{eq:impl}
    \frac{\partial \hat \delta^*_i(\theta)}{\partial \theta} \left[\nabla_{\delta}^2 \ell_i'(\theta, \hat \delta^*_i(\theta)) + c \sum_{k=1}^{2p} \gamma_k a_k a_k^\top\right] = - \nabla_{\theta \delta} \ell_i'(\theta, \hat \delta^*_i(\theta)). 
\end{align}
where we define $\gamma_k := \frac{1}{\left(b_k - a_k^\top \hat \delta^*_i(\theta))\right)^2}$. 

Further, note that $A = \big(I_p, -I_p\big)^\top$. Thus the first $p$ rows of $A$ (i.e., $a_k, k=1,...,p$) correspond to the $p$ basis vectors of $\mathbb{R}^p$, and hence $a_k a_k^\top = \operatorname{diag}(e_k)$, where $e_k$ is the $k$-th basis vector of $\mathbb{R}^p$. Thus, considering the first $p$ rows we obtain $\sum_{k=1}^{p} \gamma_k a_k a_k^\top = \operatorname{diag}(\gamma_1, ..., \gamma_p)$. Similarly, the bottom $p$ rows yields $\sum_{k=p+1}^{2p} \gamma_k a_k a_k^\top = \operatorname{diag}(\gamma_{p+1}, ..., \gamma_{2p})$. Therefore, we have 
\begin{align} \label{eq:diagc}
    c \sum_{k=1}^{2p} \gamma_k a_k a_k^\top = c \operatorname{diag}(\gamma_1 + \gamma_{p+1}, ..., \gamma_p + \gamma_{2p}) \coloneqq C_i (\theta).  
\end{align} 
Substituting \cref{eq:diagc} in \cref{eq:impl} yields 
\begin{align*} 
    \frac{\partial \hat \delta^*_i(\theta)}{\partial \theta} \left[\nabla_{\delta}^2 \ell_i'(\theta, \hat \delta^*_i(\theta)) + C_i (\theta) \right] = - \nabla_{\theta \delta} \ell_i'(\theta, \hat \delta^*_i(\theta)).
\end{align*}
If $\nabla_{\delta}^2 \ell_i'(\theta, \hat \delta^*_i(\theta)) + C_i (\theta)$ is invertable, the above equation further implies
\begin{align} \label{eq:impl2}
    \frac{\partial \hat \delta^*_i(\theta)}{\partial \theta} = - \nabla_{\theta \delta} \ell_i'(\theta, \hat \delta^*_i(\theta)) \left[\nabla_{\delta}^2 \ell_i'(\theta, \hat \delta^*_i(\theta)) + C_i (\theta) \right]^{-1}.  
\end{align}
 Now, combining \cref{eq:impl2} and \cref{eq:chain} we obtain 
\begin{align*}
    \nabla \mathcal{L} (\theta) = \frac{r}{\sum_{i=1}^M g_i (\theta, \hat \delta^*_i(\theta))}& \sum_{i=1}^M \Big(\nabla_\theta g_i (\theta, \hat \delta^*_i(\theta)) \\
    &- \nabla_{\theta \delta} \ell_i'(\theta, \hat \delta^*_i(\theta)) \left[\nabla_{\delta}^2 \ell_i'(\theta, \hat \delta^*_i(\theta)) + C_i (\theta) \right]^{-1} \nabla_\delta g_i (\theta, \hat \delta^*_i(\theta))\Big), 
\end{align*}
which completes the proof. 




\section{Convergence Analysis of the CID Algorithm} \label{app:theory}

We provide the convergence analysis of the CID algorithm for solving the generic compositional bilevel optimization problem (\ref{eq:cblo}), which we rewrite as follows:
{\small
\begin{align} \label{eq:blo}  
&\underset{\theta}\min \hspace{2pt} F(\theta):= f\left(g \left(\theta, \delta^{*}(\theta) \right)\right) = f\left(\frac{1}{M} \sum_{i=1}^{M} g_i \left(\theta, \delta_i^{*}(\theta) \right)\right) \\ 
&\text {s.t. } \delta^{*}(\theta) = \left(\delta^{*}_1(\theta), ..., \delta^{*}_M(\theta)\right) = \underset{\left(\delta_1, ..., \delta_M\right) \in \mathcal{V}_1 \times ...\times \mathcal{V}_M}{\argmin} \frac{1}{M} \sum_{i=1}^{M} h_i\left(\theta, \delta_i \right). \nonumber
\end{align}}%

{\bf Challenge and Novelty.} We note that although bilevel optimization and compositional optimization have been well studied in the optimization literature, to our best knowledge, there have not been any theoretical analysis of compositional bilevel optimization. The special challenge arising in such a problem is due to the fact that the bias error caused by the stochastic estimation of the compositional function  in the outer-loop is further complicated by the approximation error from the inner loop. Our main novel development here lies in tracking such an error in the convergence analysis. 

To proceed the analysis, we let $w = (\theta, \delta)$ denote all optimization parameters. We denote by $\big\|\cdot\big\|$ the $\ell_2$ norm for vectors and the spectral norm for matrices. 

We adopt the following assumptions for our analysis, which are widely used in bilevel and compositional optimization literature \citep{grazzi2020bo,ji2021bo,ji2021lower,wang2017stochastic,chen2021solving}. 
\begin{assum}\label{ass:lip} 
The objective functions $f$, $g_i$, and $h_i$ for any $i=1,\ldots,M$ satisfy 
\begin{list}{$\bullet$}{\topsep=0.1ex \leftmargin=0.1in \rightmargin=0.in \itemsep =-1mm}
\item $f$ is $C_f$-Lipschitz continuous and $L_f$-smooth, i.e., for any $z$ and $z'$, 
\begin{align}
    \big|f(z) - f(z')\big| \leq C_f \big\|z - z'\big\|, \quad \big\|\nabla f(z) - \nabla f(z')\big\| \leq L_f\big\|z - z'\big\|. 
\end{align} 
\item $g_i$ is $C_g$-Lipschitz continuous and $L_g$-smooth, i.e., for any $w$ and $w'$, 
\begin{align}
    \big\|g_i(w) - g_i(w')\big\| \leq C_g \big\|w - w'\big\|, \quad \big\|\nabla g_i(w) - \nabla g_i(w')\big\| \leq L_g\big\|w - w'\big\|. 
\end{align} 
\item $h_i$ is $L_h$-smooth, i.e., for any $w$ and $w'$, 
\begin{align}
   \big\|\nabla h_i(w) - \nabla h_i(w')\big\| \leq L_h\big\|w - w'\big\|. 
\end{align} 
\end{list} 
\end{assum}
\begin{assum}\label{ass:h}
The function $h_i(\theta, \delta)$ for any $i=1,\ldots,M$ is $\mu$-strongly convex w.r.t. $\delta$ and its second-order derivatives $\nabla_{\theta} \nabla_{\delta} h_i(w)$ and $\nabla_{\delta}^2 h_i(w)$ are $L_{\theta\delta}$- and $L_{\delta\delta}$-Lipschitz, i.e., for any $w$ and $w'$, 
\begin{align}
    \big\|\nabla_{\theta} \nabla_{\delta} h_i(w) - \nabla_{\theta} \nabla_{\delta} h_i(w)\big\| \leq L_{\theta\delta}\big\|w - w'\big\|, \quad \big\|\nabla_{\delta}^2 h_i(w) - \nabla_{\delta}^2 h_i(w')\big\| \leq L_{\delta\delta}\big\|w - w'\big\|.
\end{align}
\end{assum}
\begin{assum}\label{ass:bdv}
The stochastic sample $g_i$ for any $i=1,\ldots,M$ has bounded variance, i.e., 
\begin{align} 
    \mathbb{E}_i \big\|g_i(\theta, \delta_i) - \frac{1}{M}\sum_{j=1}^M g_j(\theta, \delta_j)\big\|^2 \leq \sigma_g^2. 
\end{align}
\end{assum}

The following theorem (as restatement of \Cref{thr}) characterizes the convergence rate of our designed CID algorithm.
\begin{theorem}[Re-statement of \Cref{thr}] \label{thra}
Suppose that Assumptions \ref{ass:lip}, \ref{ass:h}, \ref{ass:bdv} hold. Select the stepsizes as $\beta_t =\frac{1}{\sqrt{T}}$ and $\eta_t \in [\frac{1}{2}, 1)$, and batchsize as $|\mathcal{B}| = \mathcal{O}(T)$. Then, the iterates $\theta_t, t=0,..., T-1$ of the CID algorithm satisfy 
\begin{align}
    \frac{\sum_{t=0}^{T-1} \mathbb{E} \big\| \nabla F(\theta_t) \big\|^2}{T} \leq \mathcal{O}\Big( \frac{1}{\sqrt{T}} + (1-\alpha\mu)^K \Big) \nonumber
\end{align} 
\end{theorem} 


In the following two subsections, we first establish a number of useful supporting lemmas and then provide the proof of \Cref{thra} (which is a restatement of \Cref{thr}).

\subsection{Supporting Lemmas}

For notational convenience, we let $L = \max\{L_f, L_g, L_h\}$, $C = \max\{C_f, C_g\}$, and $\tau = \max\{L_{\theta\delta}, L_{\delta\delta}\}$. 
\begin{lemma} \label{lem:lip}
Suppose that Assumptions \ref{ass:lip} and \ref{ass:h} hold. Then, the total objective $F(\theta)$ (defined at the outer level of problem (\ref{eq:blo})  is $L_F$-smooth, i.e., for any $\theta$, $\theta'$, 
\begin{align}
    \big\|\nabla F(\theta) - \nabla F(\theta')\big\| \leq L_F\big\|\theta - \theta'\big\|, 
\end{align}
where $L_F = C^2 L\left(1 + \frac{L}{\mu}\right)^2 + CL_G$. 
\end{lemma}

\begin{proof}
Applying the chain rule, we have 
\begin{align}
    \nabla F(\theta) = \frac{\partial g \left(\theta, \delta^{*}(\theta) \right)}{\partial \theta} \nabla f\left(g \left(\theta, \delta^{*}(\theta) \right)\right).
\end{align}
Therefore, using triangle inequality, we obtain
\begin{align}
    \big\| \nabla F(\theta) - \nabla F(\theta')\big\| =& 
    \big\| \frac{\partial g \left(\theta, \delta^{*}(\theta) \right)}{\partial \theta} \nabla f\left(g \left(\theta, \delta^{*}(\theta) \right)\right) - 
    \frac{\partial g \left(\theta', \delta^{*}(\theta') \right)}{\partial \theta} \nabla f\left(g \left(\theta', \delta^{*}(\theta') \right)\right)\big\| \nonumber \\ 
    \leq & \big\| \frac{\partial g \left(\theta, \delta^{*}(\theta) \right)}{\partial \theta} \left( \nabla f\left(g \left(\theta, \delta^{*}(\theta) \right)\right) - \nabla f\left(g \left(\theta', \delta^{*}(\theta') \right)\right) \right)\big\| \nonumber \\
    & + \big\| \left( \frac{\partial g \left(\theta, \delta^{*}(\theta) \right)}{\partial \theta} - \frac{\partial g \left(\theta', \delta^{*}(\theta') \right)}{\partial \theta}\right) \nabla f\left(g \left(\theta', \delta^{*}(\theta') \right)\right)\big\| \nonumber \\ 
    \leq& \big\| \frac{\partial g \left(\theta, \delta^{*}(\theta) \right)}{\partial \theta} \big\| \big\| \nabla f\left(g \left(\theta, \delta^{*}(\theta) \right)\right) - \nabla f\left(g \left(\theta', \delta^{*}(\theta') \right)\right)\big\| \nonumber \\
    & + \big\| \frac{\partial g \left(\theta, \delta^{*}(\theta) \right)}{\partial \theta} - \frac{\partial g \left(\theta', \delta^{*}(\theta') \right)}{\partial \theta}\big\| \big\|\nabla f\left(g \left(\theta', \delta^{*}(\theta') \right)\right)\big\| \nonumber \\ 
    \leq& L_f \big\| \frac{\partial g \left(\theta, \delta^{*}(\theta) \right)}{\partial \theta} \big\| \big\| g \left(\theta, \delta^{*}(\theta) \right) - g \left(\theta', \delta^{*}(\theta') \right)\big\| \nonumber \\
    &+ C_f \big\| \frac{\partial g \left(\theta, \delta^{*}(\theta) \right)}{\partial \theta} - \frac{\partial g \left(\theta', \delta^{*}(\theta') \right)}{\partial \theta}\big\|. \label{eq:lipf1}
\end{align}
The chain rule yields 
\begin{align*}
    \frac{\partial g_i\left(\theta, \delta^{*}_i(\theta) \right)}{\partial \theta} =& 
    \nabla_{\theta} g_i\left(\theta, \delta^{*}_i(\theta) \right) + \frac{\partial \delta^{*}_i(\theta)}{\partial \theta} \nabla_{\delta} g_i\left(\theta, \delta^{*}_i(\theta) \right) \\ 
    =& \nabla_{\theta} g_i\left(\theta, \delta^{*}_i(\theta) \right) - \nabla_{\theta} \nabla_{\delta} h_i\left(\theta, \delta^{*}_i(\theta) \right) \left[ \nabla_{\delta}^2 h_i\left(\theta, \delta^{*}_i(\theta) \right)\right]^{-1} \nabla_{\delta} g_i\left(\theta, \delta^{*}_i(\theta) \right), 
\end{align*}
where the last equality follows from the implicit differentiation result for bilevel optimization \cite{pedregosa2016hyperparameter,ji2021bo}.

Thus, we obtain 
\begin{align} \label{eq:impupper}
   & \big\| \frac{\partial g_i\left(\theta, \delta^{*}_i(\theta) \right)}{\partial \theta}\big\| \nonumber \\
   &\qquad \leq \big\| \nabla_{\theta} g_i\left(\theta, \delta^{*}_i(\theta) \right) \big\| + \big\| \nabla_{\theta} \nabla_{\delta} h_i\left(\theta, \delta^{*}_i(\theta) \right) \left[ \nabla_{\delta}^2 h_i\left(\theta, \delta^{*}_i(\theta) \right)\right]^{-1} \nabla_{\delta} g_i\left(\theta, \delta^{*}_i(\theta) \right) \big\| \nonumber \\
   &\qquad \leq C_g + \frac{L}{\mu} C_g. 
\end{align}
Therfore, $g \left(\theta, \delta^{*}(\theta) \right) = \frac{1}{M} \sum_{i=1}^{M} g_i \left(\theta, \delta_i^{*}(\theta) \right))$ is Lipschitz with constant $C_G = C_g \left(1 + \frac{L}{\mu}\right)$. Further, following from Lemma 2 in \cite{ji2021bo} we obtain that $\frac{\partial g \left(\theta, \delta^{*}(\theta) \right)}{\partial \theta}$ is Lipschitz with the constant $L_G$. 
Thus, combining with \cref{eq:lipf1}, we obtain 
\begin{align}
    \big\| \nabla F(\theta) - \nabla F(\theta')\big\| \leq& 
    L_f C_G^2 \big\|\theta - \theta'\big\| + C_f L_G \big\|\theta - \theta'\big\| 
\end{align}
Rearranging the above equation completes the proof. 
\end{proof}

\begin{lemma}\label{lem:track}
    Suppose that Assumptions \ref{ass:lip} and \ref{ass:bdv} hold. Then, we have 
    \begin{equation}
        \mathbb{E}_{\mathcal{B}} \big\|u_{t+1} - g(\theta_t, \delta_{t}^K)\big\|^2 \leq (1-\eta_t) \big\| u_{t} - g(\theta_{t-1}, \delta_{t-1}^K) \big\|^2 + \frac{2\eta_t^2}{|\mathcal{B}|} \sigma_g^2 + \frac{C^2}{\eta_t} (1+\kappa^2) \big\| \theta_t - \theta_{t-1} \big\|^2.
    \end{equation}
\end{lemma} 

\begin{proof}
We first show that $\frac{\partial \delta_{i,t}^K}{\partial \theta}$ is $\kappa$-Lipschitz. To explicitly write the dependency of $\delta_{i,t}^{k}$ on $\theta_t$, we define $\delta_{i}^{k}(\theta_t) \coloneqq \delta_{i,t}^{k}$. Then we have
\begin{align*}
    &\hspace{-1cm}\big\| \delta_i^K(\theta) - \delta_i^K(\theta') \big\|  \\  
    =&\big\| \Pi_{\mathcal{X}}\left(\delta_i^{K-1}(\theta) -\alpha\nabla_\delta h_i\left(\theta, \delta_i^{K-1}(\theta)\right)\right)  - 
    \Pi_{\mathcal{X}}\left(\delta_i^{K-1}(\theta') -\alpha\nabla_\delta h_i\left(\theta', \delta_i^{K-1}(\theta')\right)\right)\big\| \\ 
    \leq& \big\| \delta_i^{K-1}(\theta) -\alpha\nabla_\delta h_i\left(\theta, \delta_i^{K-1}(\theta)\right)  - 
    \delta_i^{K-1}(\theta') + \alpha\nabla_\delta h_i\left(\theta', \delta_i^{K-1}(\theta')\right)\big\| \\ 
    \leq& \underbrace{\big\| \delta_i^{K-1}(\theta) - \delta_i^{K-1}(\theta') + \alpha \left(\nabla_\delta h_i\left(\theta', \delta_i^{K-1}(\theta')\right) -\nabla_\delta h_i\left(\theta', \delta_i^{K-1}(\theta)\right)\right)\big\|}_{T_1} \\ 
    &+ \alpha \big\|\nabla_\delta h_i\left(\theta', \delta_i^{K-1}(\theta)\right) -\nabla_\delta h_i\left(\theta, \delta_i^{K-1}(\theta)\right)\big\| \\ 
    \leq& \left(\frac{L - \mu}{L+\mu} \right)\big\| \delta_i^{K-1}(\theta) - \delta_i^{K-1}(\theta')\big\| + \alpha L\big\| \theta - \theta' \big\|, 
\end{align*}
where we upper-bound the term $T_1$ using the fact that the operator $y \rightarrow y - \alpha \nabla h(y)$ is a contraction mapping with the constant $\frac{L - \mu}{L+\mu}$ for an $L$-smooth and $\mu$-stongly convex function $h$ when the stepsize $\alpha$ is set to $\frac{2}{L+\mu}$. 
Hence, telescoping the previous inequality over $k$ from $K-1$ down to $0$ yields 
\begin{align}
    \big\| \delta_i^K(\theta) - \delta_i^K(\theta') \big\| \leq& 
    \left(\frac{L - \mu}{L+\mu} \right)^K \big\| \delta_i^{0}(\theta) - \delta_i^{0}(\theta')\big\| + \alpha L\big\| \theta - \theta' \big\| \sum_{k=0}^{K-1} \left(\frac{L - \mu}{L+\mu} \right)^k \nonumber \\ 
    \leq&  0 + \frac{\alpha L}{1 - \frac{L - \mu}{L+\mu}} \big\| \theta - \theta' \big\| = \kappa \big\| \theta - \theta' \big\|, 
\end{align}
where the second inequality follows because $\delta_i^{0}(\theta) = \delta_i^{0}(\theta')$ as the same initial point, and the last equality follows by setting the stepsize $\alpha$ to $\frac{2}{L+\mu}$. 

Denote $d_t = (1-\eta_t) \left(g(\theta_t, \delta_t^K) - g(\theta_{t-1}, \delta_{t-1}^K)\right) = \frac{1-\eta_t}{M} \sum_{i=1}^M \left(g_i(\theta_t, \delta_{i,t}^K) - g_i(\theta_{t-1}, \delta_{i,t-1}^K)\right)$. We can then obtain
\begin{align}
    \big\|d_t\big\|^2 \leq& \frac{(1-\eta_t)^2}{M} \sum_{i=1}^M \big\| g_i(\theta_t, \delta_{i,t}^K) - g_i(\theta_{t-1}, \delta_{i,t-1}^K) \big\|^2 \nonumber \\ 
    \leq& \frac{(1-\eta_t)^2}{M} \sum_{i=1}^M C^2 \left( \big\| \theta_t - \theta_{t-1} \big\|^2 + \big\| \delta_{i,t}^K - \delta_{i,t-1}^K \big\|^2 \right) \nonumber \\
    \leq& \frac{(1-\eta_t)^2}{M} \sum_{i=1}^M C^2 (1+\kappa^2)\big\| \theta_t - \theta_{t-1} \big\|^2 \nonumber \\
    =& (1-\eta_t)^2 (1+\kappa^2) C^2 \big\| \theta_t - \theta_{t-1} \big\|^2. 
\end{align}
Recall $u_{t+1} = (1-\eta_t)u_t + \eta_t g(\theta_t, \delta_{t}^K; \mathcal{B})$. Thus combining with the definition of $d_t$, we have 
\begin{align}
    &\hspace{-1cm}\mathbb{E}_{\mathcal{B}} \big\|u_{t+1} - g(\theta_t, \delta_{t}^K) + d_t\big\|^2 \nonumber \\ 
    =&\mathbb{E}_{\mathcal{B}} \big\| (1-\eta_t)\left(u_{t} - g(\theta_{t-1}, \delta_{t-1}^K)\right) + \eta_t\left(g(\theta_t, \delta_{t}^K; \mathcal{B}) - g(\theta_t, \delta_{t}^K)\right)\big\|^2 \nonumber \\ 
    =& (1-\eta_t)^2 \big\| u_{t} - g(\theta_{t-1}, \delta_{t-1}^K) \big\|^2 + \eta_t^2 \mathbb{E}_{\mathcal{B}} \big\|g(\theta_t, \delta_{t}^K; \mathcal{B}) - g(\theta_t, \delta_{t}^K)\big\|^2 \nonumber \\ 
    &+ 2(1-\eta_t)\eta_t \left< u_{t} - g(\theta_{t-1}, \delta_{t-1}^K), \mathbb{E}_{\mathcal{B}} \left(g(\theta_t, \delta_{t}^K; \mathcal{B}) - g(\theta_t, \delta_{t}^K) \right)\right> \nonumber \\ 
    =& (1-\eta_t)^2 \big\| u_{t} - g(\theta_{t-1}, \delta_{t-1}^K) \big\|^2 + \frac{\eta_t^2}{|\mathcal{B}|} \mathbb{E}_{i} \big\|g_i(\theta_t, \delta_{i,t}^K) - g(\theta_t, \delta_{t}^K)\big\|^2 \nonumber \\ 
    \leq& (1-\eta_t)^2 \big\| u_{t} - g(\theta_{t-1}, \delta_{t-1}^K) \big\|^2 + \frac{\eta_t^2}{|\mathcal{B}|} \sigma_g^2. 
\end{align} 
Based on the inequality $\big\|a+b\big\|^2 \leq (1+c) \big\|a\big\|^2 + (1+\frac{1}{c}) \big\|b\big\|^2$ for any $c>0$, by letting $c=\eta_t$, we have  
\begin{align}
    \mathbb{E}_{\mathcal{B}} \big\|u_{t+1} - g(\theta_t, \delta_{t}^K)\big\|^2 \leq& 
    (1+\eta_t) \mathbb{E}_{\mathcal{B}} \big\|u_{t+1} - g(\theta_t, \delta_{t}^K) + d_t\big\|^2 +(1 + \frac{1}{\eta_t}) \mathbb{E}_{\mathcal{B}} \big\|d_t\big\|^2 \nonumber \\ 
    \leq& (1+\eta_t) (1-\eta_t)^2 \big\| u_{t} - g(\theta_{t-1}, \delta_{t-1}^K) \big\|^2 + \frac{(1+\eta_t)\eta_t^2}{|\mathcal{B}|} \sigma_g^2 \nonumber \\ 
    &+ \frac{1+\eta_t}{\eta_t} (1-\eta_t)^2 (1+\kappa^2) C^2 \big\| \theta_t - \theta_{t-1} \big\|^2 \nonumber \\ 
    \leq& (1-\eta_t) \big\| u_{t} - g(\theta_{t-1}, \delta_{t-1}^K) \big\|^2 + \frac{2\eta_t^2}{|\mathcal{B}|} \sigma_g^2 + \frac{C^2}{\eta_t} (1+\kappa^2) \big\| \theta_t - \theta_{t-1} \big\|^2. 
\end{align}
Hence, the proof is complete. 
\end{proof}

\begin{lemma}\label{lem:bi}
Suppose that Assumptions \ref{ass:lip} and \ref{ass:h} hold. Then we have 
\begin{align}
    \Big\| \frac{\partial g \left(\theta_t, \delta^{*}(\theta_t) \right)}{\partial \theta} - \widehat \nabla g(\theta_t, \delta_t^K)\Big\|^2 \leq \Omega (1-\alpha\mu)^K\Delta_0, 
\end{align} 
where $\Delta_0 = \max_{i,t} \big\| \delta^{*}_i(\theta_t) - \delta_{0} \big\|^2$ and $\Omega = \mathcal{O} \Big( L + \frac{\tau^2 C^2}{\mu^2} + L\left(\kappa + \frac{\tau C}{\mu^2} \right)^2 \Big)$.
\end{lemma} 
\begin{proof}
    The proof follows the steps similar to those in the proof of Lemma 3 in \cite{ji2021bo}. 
\end{proof}
In the following, we define $\Lambda = \Omega(1-\alpha\mu)^K\Delta_0$.  

\begin{lemma}
Suppose that Assumptions \ref{ass:lip}, \ref{ass:h}, \ref{ass:bdv} hold. Then, we have 
\begin{align}
    \mathbb{E}_\mathcal{B} F(\theta_{t+1}) - F(\theta_{t}) 
    \leq& -\beta_t \alpha_t \big\| \nabla F(\theta_t) \big\|^2 + \beta_t \Gamma \Delta_0 (1-\alpha\mu)^K \nonumber \\ 
    &+ \eta_t \mathbb{E}_\mathcal{B} \big\|g(\theta_t, \delta_t^K) - u_{t+1}\big\|^2 + \frac{L_F\beta_t^2}{2} C^4\left(1+\frac{L}{\mu}\right)^2, \label{eq:desc}
\end{align}
where $\alpha_t = \frac{1}{2} -  \frac{\beta_t L^2}{\eta_t} C^2\left(1+\frac{L}{\mu}\right)^2$. 
\end{lemma}
\begin{proof}
Based on the Lipschitzness of $\nabla F(\theta)$ in \Cref{lem:lip}, we have 
\begin{align}
    F(\theta_{t+1}) - F(\theta_{t}) 
    \leq& \left< \nabla F(\theta_t), \theta_{t+1} - \theta_{t}\right> + \frac{L_F}{2} \big\| \theta_{t+1} - \theta_{t}\big\|^2 \nonumber \\ 
    \leq& -\beta_t \big\|\nabla F(\theta_t)\big\|^2 +\beta_t  \left<\nabla F(\theta_t), \nabla F(\theta_t) - \widehat{\nabla} g(\theta_t, \delta_t^K; \mathcal{B}) \nabla f(u_{t+1})\right> \nonumber \\ 
    &+ \frac{L_F\beta_t^2}{2} \big\| \widehat{\nabla} g(\theta_t, \delta_t^K; \mathcal{B}) \nabla f(u_{t+1}) \big\|^2 \nonumber \\ 
    \leq& -\beta_t \big\|\nabla F(\theta_t)\big\|^2 + 
    \underbrace{\beta_t  \left<\nabla F(\theta_t), \nabla F(\theta_t) - \widehat{\nabla} g(\theta_t, \delta_t^K) \nabla f\left(g(\theta_t, \delta_t^K)\right)\right>}_{A_1} \nonumber \\ 
    &+ \underbrace{\beta_t  \left<\nabla F(\theta_t), \widehat{\nabla} g(\theta_t, \delta_t^K) \nabla f\left(g(\theta_t, \delta_t^K)\right) - \widehat{\nabla} g(\theta_t, \delta_t^K; \mathcal{B}) \nabla f(u_{t+1})\right>}_{A_2} \nonumber \\ 
    & +\frac{L_F\beta_t^2}{2} \big\| \widehat{\nabla} g(\theta_t, \delta_t^K; \mathcal{B}) \nabla f(u_{t+1}) \big\|^2. \label{eq:desc1}
\end{align}
Next, we upper-bound the inner product terms $A_1$ and $A_2$, respectively. Using Young's inequality, we obtain 
\begin{align}
    A_1 \leq& \frac{\beta_t}{2} \big\| \nabla F(\theta_t) \big\|^2 + \frac{\beta_t}{2} \big\| \nabla F(\theta_t) - \widehat{\nabla} g(\theta_t, \delta_t^K) \nabla f\left(g(\theta_t, \delta_t^K)\right) \big\|^2 \nonumber \\ 
    \leq& \frac{\beta_t}{2} \big\| \nabla F(\theta_t) \big\|^2 + \beta_t \big\|\frac{\partial g \left(\theta_t, \delta^{*}(\theta_t) \right)}{\partial \theta}\big\|^2 \big\| \nabla f\left(g \left(\theta_t, \delta^{*}(\theta_t) \right)\right) - \nabla f\left(g(\theta_t, \delta_t^K)\right) \big\|^2 \nonumber \\ 
    &+ \beta_t \big\|\nabla f\left(g(\theta_t, \delta_t^K)\right)\big\|^2 \big\| \frac{\partial g \left(\theta_t, \delta^{*}(\theta_t) \right)}{\partial \theta} - \widehat \nabla g(\theta_t, \delta_t^K)\big\|^2 \nonumber \\ 
    \leq& \frac{\beta_t}{2} \big\| \nabla F(\theta_t) \big\|^2 +  \beta_t L_G^2 L^2 \big\| g \left(\theta_t, \delta^{*}(\theta_t) \right) - g(\theta_t, \delta_t^K) \big\|^2 \nonumber \\ 
    &+ \beta_t C^2 \big\| \frac{\partial g \left(\theta_t, \delta^{*}(\theta_t) \right)}{\partial \theta} - \widehat \nabla g(\theta_t, \delta_t^K)\big\|^2 \nonumber \\ 
    \leq& \frac{\beta_t}{2} \big\| \nabla F(\theta_t) \big\|^2 
    + \frac{\beta_t L_G^2 L^2}{M} \sum_{i=1}^{M}\big\| g_i \left(\theta_t, \delta^{*}_i(\theta_t) \right) - g_i(\theta_t, \delta_{i,t}^K) \big\|^2 + \beta_t C^2 \Lambda \nonumber \\ 
    \leq& \frac{\beta_t}{2} \big\| \nabla F(\theta_t) \big\|^2 
    + \frac{\beta_t L_G^2 L^2 C^2}{M} \sum_{i=1}^{M}\big\| \delta^{*}_i(\theta_t) - \delta_{i,t}^K \big\|^2 + \beta_t C^2 \Lambda \nonumber \\ 
    \leq& \frac{\beta_t}{2} \big\| \nabla F(\theta_t) \big\|^2 
    + \beta_t L_G^2 L^2 C^2 \frac{(1-\alpha\mu)^K}{M} \sum_{i=1}^{M}\big\| \delta^{*}_i(\theta_t) - \delta_{0} \big\|^2 + \beta_t C^2 \Lambda \nonumber \\ 
    \leq& \frac{\beta_t}{2} \big\| \nabla F(\theta_t) \big\|^2 
    + \beta_t L_G^2 L^2 C^2 \Delta_0 (1-\alpha\mu)^K  + \beta_t C^2 \Omega(1-\alpha\mu)^K\Delta_0 \nonumber \\ 
    =& \frac{\beta_t}{2} \big\| \nabla F(\theta_t) \big\|^2 
    + \beta_t \Gamma \Delta_0 (1-\alpha\mu)^K, \label{eq:a1}
\end{align}
where $\Gamma = L_G^2 L^2 C^2 + C^2\Omega$, $\Delta_0 = \max_{i,t} \big\| \delta^{*}_i(\theta_t) - \delta_{0} \big\|^2$, and $\Lambda = \Omega(1-\alpha\mu)^K\Delta_0$. 

Further, we have 
\begin{align}
    \mathbb{E}_\mathcal{B} A_2 =& \beta_t \mathbb{E}_\mathcal{B} \left< \nabla F(\theta_t), \widehat{\nabla} g(\theta_t, \delta_t^K; \mathcal{B}) \nabla f\left(g(\theta_t, \delta_t^K)\right) - \widehat{\nabla} g(\theta_t, \delta_t^K; \mathcal{B}) \nabla f(u_{t+1}) \right> \nonumber \\ 
    \leq& \beta_t \big\|\nabla F(\theta_t)\big\| \mathbb{E}_\mathcal{B} \left[\big\|\widehat{\nabla} g(\theta_t, \delta_t^K; \mathcal{B})\big\| \big\|\nabla f\left(g(\theta_t, \delta_t^K)\right) - \nabla f(u_{t+1})\big\| \right] \nonumber \\ 
    \leq& \beta_t L \big\|\nabla F(\theta_t)\big\| \mathbb{E}_\mathcal{B} \left[\big\|\widehat{\nabla} g(\theta_t, \delta_t^K; \mathcal{B})\big\| \big\|g(\theta_t, \delta_t^K) - u_{t+1}\big\| \right] \nonumber \\ 
    \leq& \eta_t \mathbb{E}_\mathcal{B} \big\|g(\theta_t, \delta_t^K) - u_{t+1}\big\|^2 + \frac{\beta_t^2 L^2}{\eta_t} \big\| \nabla F(\theta_t) \big\|^2 \mathbb{E}_\mathcal{B} \big\|\widehat{\nabla} g(\theta_t, \delta_t^K; \mathcal{B})\big\|^2 \nonumber \\ 
    \leq& \eta_t \mathbb{E}_\mathcal{B} \big\|g(\theta_t, \delta_t^K) - u_{t+1}\big\|^2 + \frac{\beta_t^2 L^2}{\eta_t} C^2(1+\frac{L}{\mu})^2 \big\| \nabla F(\theta_t) \big\|^2, \label{eq:a2}
\end{align}
where the last inequality uses the upper-bound $\big\|\widehat \nabla g_i(\theta_t, \delta_{i,t}^K)\big\| \leq C + \frac{L}{\mu} C$, which can be obtained similarly to \cref{eq:impupper}. 

Therefore, taking the conditional expectation $\mathbb{E}_\mathcal{B}$ in both sides of \cref{eq:desc1}, applying the bounds for $A_1$ and $\mathbb{E}_\mathcal{B} A_2$ in \cref{eq:a1,eq:a2}, and noting that $\mathbb{E}_\mathcal{B} \big\| \widehat{\nabla} g(\theta_t, \delta_t^K; \mathcal{B}) \nabla f(u_{t+1}) \big\|^2 \leq C^2 \mathbb{E}_\mathcal{B} \big\| \widehat{\nabla} g(\theta_t, \delta_t^K; \mathcal{B})\big\|^2 \leq C^4\left(1+\frac{L}{\mu}\right)^2$, we obtain 
\begin{align*}
    \mathbb{E}_\mathcal{B} F(\theta_{t+1}) - F(\theta_{t}) 
    \leq& -\frac{\beta_t}{2} \big\| \nabla F(\theta_t) \big\|^2 + \beta_t \Gamma \Delta_0 (1-\alpha\mu)^K + \eta_t \mathbb{E}_\mathcal{B} \big\|g(\theta_t, \delta_t^K) - u_{t+1}\big\|^2 \nonumber \\ 
    &+ \frac{\beta_t^2 L^2}{\eta_t} C^2\left(1+\frac{L}{\mu}\right)^2 \big\| \nabla F(\theta_t) \big\|^2 + \frac{L_F\beta_t^2}{2} C^4\left(1+\frac{L}{\mu}\right)^2 \nonumber \\ 
    \leq& -\beta_t \left(\frac{1}{2} -  \frac{\beta_t L^2}{\eta_t} C^2\left(1+\frac{L}{\mu}\right)^2 \right) \big\| \nabla F(\theta_t) \big\|^2 + \beta_t \Gamma \Delta_0 (1-\alpha\mu)^K \nonumber \\ 
    &+ \eta_t \mathbb{E}_\mathcal{B} \big\|g(\theta_t, \delta_t^K) - u_{t+1}\big\|^2 + \frac{L_F\beta_t^2}{2} C^4\left(1+\frac{L}{\mu}\right)^2. 
\end{align*}
Then, the proof is complete. 
\end{proof}


\subsection{Proof of \Cref{thra} (i.e., \Cref{thr})}
Denote $V_t = F(\theta_t) + \big\|g(\theta_{t-1}, \delta_{t-1}^K) - u_{t}\big\|^2$. Then, using \cref{eq:desc} we obtain 
\begin{align}
    \mathbb{E}_\mathcal{B} V_{t+1} - V_{t} 
    \leq& -\beta_t \alpha_t \big\| \nabla F(\theta_t) \big\|^2 + \beta_t \Gamma \Delta_0 (1-\alpha\mu)^K - \big\|g(\theta_{t-1}, \delta_{t-1}^K) - u_{t}\big\|^2 \nonumber \\ 
    &+ (1+\eta_t) \mathbb{E}_\mathcal{B} \big\|g(\theta_t, \delta_t^K) - u_{t+1}\big\|^2 + \frac{1}{2} L_F\beta_t^2 C^4\left(1+\frac{L}{\mu}\right)^2 \nonumber \\ 
    \leq& -\beta_t \alpha_t \big\| \nabla F(\theta_t) \big\|^2 + \beta_t \Gamma \Delta_0 (1-\alpha\mu)^K - \big\|g(\theta_{t-1}, \delta_{t-1}^K) - u_{t}\big\|^2 \nonumber \\ 
    & + (1+\eta_t)(1-\eta_t) \big\|g(\theta_{t-1}, \delta_{t-1}^K) - u_{t}\big\|^2 + \frac{2(1+\eta_t)}{|\mathcal{B}|} \eta_t^2 \sigma_g^2 \nonumber \\ 
    &+ \frac{C^2}{\eta_t} (1+\eta_t)(1+\kappa^2) \beta_t^2 C^4\left(1+\frac{L}{\mu}\right)^2 + \frac{1}{2} L_F\beta_t^2 C^4\left(1+\frac{L}{\mu}\right)^2, 
\end{align}
where the last inequality follows from \cref{lem:track}. 
Further, following from the fact that $(1-\eta_t)(1+\eta_t)=1-\eta_t^2<1$, we obtain 
\begin{align}
     \mathbb{E}_\mathcal{B} V_{t+1} - V_{t} 
    \leq& -\beta_t \alpha_t \big\| \nabla F(\theta_t) \big\|^2 + \beta_t \Gamma \Delta_0 (1-\alpha\mu)^K + \frac{2(1+\eta_t)}{|\mathcal{B}|} \eta_t^2 \sigma_g^2 \nonumber \\ 
    &+ \frac{1+\eta_t}{\eta_t} \beta_t^2 C^6 (1+\kappa^2)\left(1+\frac{L} {\mu}\right)^2 + \frac{1}{2} L_F\beta_t^2 C^4\left(1+\frac{L}{\mu}\right)^2. \label{eq:th1}
\end{align}
Now, select $\eta_t \in [\frac{1}{2}, 1)$ and $\beta_t$ such that $\alpha_t \geq \frac{1}{4}$, i.e., $\beta_t \leq \frac{1}{2 L_F^2 C^2 \left(1+\frac{L}{\mu}\right)^2}$. Hence, taking total expectation of \cref{eq:th1} yields 
\begin{align}
    \mathbb{E} V_{t+1} - \mathbb{E} V_{t} 
    \leq& - \frac{\beta_t}{4} \mathbb{E} \big\| \nabla F(\theta_t) \big\|^2 + \beta_t \Gamma \Delta_0 (1-\alpha\mu)^K + \frac{4\sigma_g^2}{|\mathcal{B}|} \nonumber \\ 
    &+ 4 \beta_t^2 C^6 (1+\kappa^2)\left(1+\frac{L} {\mu}\right)^2 + \frac{1}{2} L_F\beta_t^2 C^4\left(1+\frac{L}{\mu}\right)^2 \nonumber \\ 
    =& - \frac{\beta_t}{4} \big\| \nabla F(\theta_t) \big\|^2 + \beta_t \Gamma \Delta_0 (1-\alpha\mu)^K + \frac{4\sigma_g^2}{|\mathcal{B}|} + \beta_t^2 D_{\kappa}, \label{eq:th2}
\end{align}
where we define $D_{\kappa} = \left(4C^2(1+\kappa^2) + \frac{1}{2} L_F\right) \left(1+\kappa\right)^2 C^4$. 
Therefore, telescoping \cref{eq:th2} over $t$ from $0$ to $T-1$ yields 
\begin{align*}
    \mathbb{E} V_{T} - V_{0} 
    \leq& - \sum_{t=0}^{T-1}\frac{\beta_t}{4} \mathbb{E} \big\| \nabla F(\theta_t) \big\|^2 + \frac{4\sigma_g^2 T}{|\mathcal{B}|} + \Gamma \Delta_0 (1-\alpha\mu)^K \sum_{t=0}^{T-1} \beta_t + D_{\kappa} \sum_{t=0}^{T-1} \beta_t^2 .
\end{align*} 
Thus, rearranging terms, we obtain 
\begin{align}
    \frac{\sum_{t=0}^{T-1}\beta_t \mathbb{E} \big\| \nabla F(\theta_t) \big\|^2}{\sum_{t=0}^{T-1} \beta_t} \leq \frac{16\sigma_g^2 T}{|\mathcal{B}| \sum_{t=0}^{T-1} \beta_t} + 4\Gamma \Delta_0 (1-\alpha\mu)^K + 4 D_{\kappa} \frac{\sum_{t=0}^{T-1} \beta_t^2}{\sum_{t=0}^{T-1} \beta_t} + \frac{4V_0}{\sum_{t=0}^{T-1} \beta_t}. 
\end{align}
Hence, the proof is complete by choosing the batchsize $|\mathcal{B}| = \mathcal{O}(T)$ and stepsize $\beta_t = \frac{1}{\sqrt{T}}$. 


\end{document}